\documentclass[sigconf]{acmart}
\def\thisismainpaper{0}   

\usepackage{amsmath}
\usepackage{multirow}
\usepackage{epsfig}
\usepackage{mathrsfs}
\usepackage{hyperref}
\usepackage{dsfont}
\usepackage{array}
\usepackage{amsthm}
\usepackage{booktabs}
\usepackage{blkarray}
\usepackage{enumerate}
\usepackage{url}
\usepackage{chemarrow}
\usepackage{float}
\usepackage{color}
\usepackage{subcaption}
\usepackage[lined,ruled,linesnumbered,noend]{algorithm2e}
\SetKwRepeat{Do}{do}{while}%
\usepackage{epsf,psfrag}
\usepackage{epsfig}
\usepackage{dsfont}
\usepackage{bm}
\usepackage{aligned-overset}
\if\thisismainpaper 1
\fi

\theoremstyle{definition}
\newtheorem{theorem}{Theorem}[section]
\newtheorem{lemma}[theorem]{Lemma}
\newtheorem{corollary}[theorem]{Corollary}

\newtheorem{definition}{Definition}

\newcommand{\E}{{\rm I\kern-.3em E}}

\newcommand{\EE}[1]{\mathbb{E}\!\left[#1\right]}
\def\E{\mathbb{E}}

\def\xbar{\overline{\bm{x}}}
\def\xbart{\overline{\bm{x}}^{(t)}}
\def\xbars{\overline{\bm{x}}^{(s)}}
\def\zit{\bm{z}_i^{(t)}}
\def\zis{\bm{z}_i^{(s)}}

\def\zbf{\bm{z}}
\def\xbf{\bm{x}}

\def\maxx{\max_{m}\norm{\bm{x}_m^{(0)}}}
\def\gis{g_i(\zbf_i^{(s)}; \xi_i^{(s)})}

\def\gjs{g_j(\zbf_j^{(s)}; \xi_j^{(s)})}

\def\nabfit{\nabla f_i(\zbf_i^{(t)})}
\def\nabfixbar{\nabla f_i(\xbart)}
\def\nabfxbar{\nabla f(\xbart)}
\def\maxgs{\max_{j}\norm{g_j(\zbf_j^{(s)}; \xi_j^{(s)}) }}

\def\nabfjs{\nabla f_j(\zbf_j^{(s)})}
\def\nabfjt{\nabla f_j(\zbf_j^{(t)})}

\def\zjt{\zbf_j^{(t)}}
\def\qit{Q_i^{(t)}}
\def\qis{Q_i^{(s)}}
\def\mt{M^{(t)}}
\def\ms{M^{(s)}}

\def\Rd{\mathds{R}^{d}}
\newcommand\norm[1]{\left\lVert#1\right\rVert}

\if\thisismainpaper 1
\fi

\title{
Optimizing Stochastic Gradient Push under Broadcast Communications
}

\author{Tuan Nguyen}
\affiliation{%
  \institution{Pennsylvania State University}
  \city{State College, PA}
  \country{USA}}
\email{tmn5319@psu.edu}

\author{Ting He}
\affiliation{%
  \institution{Pennsylvania State University}
  \city{State College, PA}
  \country{USA}}
\email{tinghe@psu.edu}

\begin{document}
\begin{abstract}
We consider the problem of minimizing the convergence time for decentralized federated learning (DFL) in wireless networks under broadcast communications, with focus on mixing matrix design. The mixing matrix is a critical hyperparameter for DFL that simultaneously controls the convergence rate across iterations and the communication demand per iteration, both strongly influencing the convergence time. Although the problem has been studied previously, existing solutions are mostly designed for decentralized parallel stochastic gradient descent (D-PSGD), which requires the mixing matrix to be symmetric and doubly stochastic. These constraints confine the activated communication graph to undirected (i.e., bidirected) graphs, which limits design flexibility. In contrast, we consider mixing matrix design for stochastic gradient push (SGP), which allows asymmetric mixing matrices and hence directed communication graphs. By analyzing how the convergence rate of SGP depends on the mixing matrices, we extract an objective function that explicitly depends on graph-theoretic parameters of the activated communication graph, based on which we develop an efficient design algorithm with performance guarantees. Our evaluations based on real data show that the proposed solution can notably reduce the convergence time compared to the state of the art without compromising the quality of the trained model. 
 \looseness=-1
\end{abstract}



\settopmatter{printfolios=true}

\maketitle

\section{Introduction}

\emph{Decentralized federated learning (DFL)}~\cite{Lian17NIPS} is an emerging machine learning paradigm that allows distributed learning agents to collaboratively learn a shared model from the union of their local data without directly sharing the data or relying on centralized parameter servers. Compared to centralized {federated learning (FL)}~\cite{McMahan17AISTATS}, DFL agents  directly exchange model updates with their neighbors, which leads to better robustness by avoiding a single point of failure and better load balancing by spreading the communication loads across nodes without increasing the computational complexity~\cite{Lian17NIPS}. 

Meanwhile, DFL still faces significant challenges in bandwidth-limited edge networks. In such networks, communication cost often dominates the total training cost~\cite{chen2022federated}, triggering the need to reduce the cost of repeated parameter sharing between learning agents without compromising convergence. The problem has attracted significant attention from the research community, which proposed solutions that focused on either reducing the amount of data per communication through compression methods (e.g., \cite{Compression1}) or reducing the number of communications through hyperparameter optimization (e.g., \cite{Chiu23JSAC,Herrera25OJCS}) or adaptive communications (e.g., \cite{Singh20CDC}). In this work, we focus on hyperparameter optimization by designing a critical hyperparameter called \emph{mixing matrix}. Our solution can be combined with compression methods for further improvement. \looseness=-1

As a square matrix containing the weights used in local parameter aggregation, the mixing matrix plays a crucial role of controlling both the convergence rate in the number of iterations and the inter-agent communication demand per iteration, where each non-zero off-diagonal entry \emph{activates} an inter-agent communication. While a denser mixing matrix can better approximate the ideal mixing matrix for perfect consensus and hence allows DFL to converge in fewer iterations, it also activates more inter-agent communications that increases the wall-clock time per iteration, causing a nontrivial tradeoff. There is thus an active line of work on designing mixing matrices to minimize the total wall-clock time until convergence. Most existing works assumed simplistic measures of the communication time, e.g., the maximum degree~\cite{hua2022efficient,le2023refined} or the minimum number of matchings~\cite{MATCHA22,Chiu23JSAC} of the activated communication graph. While proportional to the communication time in wired networks under certain assumptions (e.g., all communications take equal time), such simplistic measures cannot capture the communication time in wireless networks with broadcast communications. To our knowledge, this problem only started to be addressed recently in \cite{Herrera25OJCS}, which proposed design algorithms to minimize the total number of \emph{collision-free transmission slots} until convergence. \looseness=-1
In this work, we consider mixing matrix design for minimizing the convergence time, measured by the total number of collision-free transmission slots as in \cite{Herrera25OJCS}, under broadcast communications. However, instead of using decentralized parallel stochastic gradient descent (D-PSGD) \cite{Lian17NIPS} as the learning algorithm as in \cite{Herrera25OJCS}, we adopt stochastic gradient push (SGP)~\cite{Assran19ICML}, which allows asymmetric mixing matrices and hence directed inter-agent communications. Our solution is built upon a new convergence theorem for SGP that extracts the convergence impact of mixing matrices into an explicit function of simple graph-theoretical parameters of the activated communication graph, which then allows tractable algorithm design. To our knowledge, this is the \emph{first work} that addresses the optimization of mixing matrices for SGP. 

\vspace{-0.5em}

\subsection{Related Works}

\textbf{Decentralized federated learning.} 
First explored by \cite{Lian17NIPS} through Decentralized Parallel Stochastic Gradient Descent (D-PSGD), DFL removes the central server in \cite{McMahan17AISTATS} and enables model training over peer-to-peer networks. 
Subsequent studies such as \cite{D2ICML18,ICMLhonor,Xin21} advance DFL in both algorithms and theories, with the main focus on improving the convergence rate measured in \#iterations.

\textbf{Communication cost reduction.} 
There are two general approaches for reducing the communication cost: reducing the cost per communication through compression, e.g., \cite{Compression1,Compression2,Compression3}, 
and reducing the number of communications, e.g.,  \cite{sysml19,Ngu19INFOCOM,Wang19JSAC}. 
The two approaches can be combined for further improvement \cite{Singh20CDC,Singh21JSAIT}. 
Instead of either activating all the links or activating none, it has been observed that better performance can be achieved by activating suitable subsets of links. To choose such subsets, \cite{Singh20CDC,Singh21JSAIT} proposed an event-triggered mechanism and \cite{MATCHA22,Chiu23JSAC} proposed to design randomized activation strategies with optimized probabilities. The latter approach was then extended to wireless networks by activating nodes instead of links under broadcast communications and interference constraints \cite{Herrera25OJCS}. 
In this work, we aim at optimizing the activated communication topology under broadcast communications and interference constraints as in \cite{Herrera25OJCS}, with the objective of minimizing the total wall-clock time until convergence. While minimizing the convergence time has been studied in a number of works \cite{MATCHA22,Chiu23JSAC,hua2022efficient,le2023refined}, doing so in a wireless setting with broadcast communications and interference was only addressed recently in \cite{Herrera25OJCS}, which designed \emph{undirected} (i.e., bidirected) communication graphs to minimize the total number of transmission slots under collision-free constraints. We adopt the same objective as \cite{Herrera25OJCS}, but enlarge the solution space to arbitrary \emph{directed} communication graphs by using SGP as the learning algorithm. As SGP uses the Push-Sum algorithm to correct bias \cite{Assran19ICML}, it allows asymmetric inter-agent parameter sharing, which is particularly beneficial in wireless networks due to effects like hidden terminals. As shown later (Section~\ref{sec:Performance Evaluation}), the enlarged solution space allows our solution to reach the same level of convergence in fewer transmission slots than \cite{Herrera25OJCS}, under the same interference constraints. 
While a few other works have also considered broadcast-based DFL \cite{Chen23SPAWC,Chiu23JSAC,Zhang26INFOCOM}, they had different objectives (e.g., maximizing \#successful links~\cite{Chen23SPAWC} or minimizing the energy consumption~\cite{Chiu23JSAC,Zhang26INFOCOM}). 

\textbf{Mixing matrix design in DFL.} The mixing matrix, i.e., the matrix containing the local aggregation weights, is a critical hyperparameter in DFL that controls the tradeoff between the convergence rate (in \#iterations) and the amount of communications per iteration. 
For D-PSGD~\cite{Lian17NIPS}, the impact of mixing matrices on the convergence rate is usually captured through the spectral gap~\cite{Lian17NIPS,neglia2020decentralized,jiang2023joint} or equivalent parameters~\cite{MATCHA22,Chiu23JSAC,Zhang26INFOCOM}, although recent studies have pointed out additional parameters that can also affect convergence, such as the effective number of neighbors~\cite{vogels2022beyond} and the neighborhood heterogeneity~\cite{le2023refined}. Based on the identified convergence parameters, several mixing matrix designs have been proposed to optimize the tradeoff between the convergence rate and the cost per iteration for D-PSGD~\cite{MATCHA22,Chiu23JSAC,Herrera25OJCS,hua2022efficient,le2023refined}. 
In comparison, the impact of mixing matrices on the convergence rate of SGP is  less understood, where existing works only analyzed the convergence rate for given mixing matrices \cite{Assran19ICML,Spiridonoff20JMLR, Zhou25TCNS}, but did not optimize the mixing matrices. This work fills the gap by addressing \emph{the mixing matrix design for SGP} for the first time, with the objective of minimizing the convergence time measured by the total number of transmission slots under collision-free constraints.    
While there are other training algorithms that allow asymmetric communications (e.g., stochastic push-pull~\cite{You25arXiv} and push-diging~\cite{Liang25JO}), we leave the mixing matrix design problem therein to future work. 



\subsection{Summary of Contributions}

We consider the mixing matrix design for broadcast-based DFL via SGP, with the objective of minimizing the wall-clock time until convergence measured by the total number of collision-free transmission slots. Our contributions are: 

1) We derive a new convergence theorem for SGP that explicitly summarizes the dependency of the number of iterations to reach convergence on the mixing matrices.

2) Based on the theorem, we extract a tractable objective for mixing matrix design as a closed-form function of several graph-theoretic parameters of the activated communication graph, including the maximum in/out-degree and the diameter. 

3) Based on the extracted objective function, we develop an efficient design algorithm that can construct a strongly-connected directed communication graph with guaranteed performance in terms of the objective value. 

4) We evaluate the proposed solution against benchmarks based on real wireless network topology and training data. Our results show that efficiently utilizing asymmetric parameter sharing as in the proposed solution can notably reduce the convergence time (by 11--45\%) compared to only using symmetric parameter exchanges as in existing works, without compromising the accuracy of the trained model.

\textbf{Roadmap.} Section~\ref{sec:Background and Formulation} provides background information and the problem formulation, Section~\ref{sec:Convergence Analysis} presents our convergence theorem, based on which Section~\ref{sec:Design Objective} extracts a tractable design objective, Section~\ref{sec:Algorithm} presents our design algorithm, Section~\ref{sec:Performance Evaluation} provides the performance evaluation, and Section~\ref{sec:Conclusion} concludes the paper.  
\if\thisismainpaper 1
\textbf{All the proofs are provided in Appendix~A.1--A.2 of \cite{Nguyen26arXiv_SGP}}.
\else
\textbf{All the proofs are provided in Appendix~\ref{appendix:Convergence Proof}--\ref{appendix:Proofs}}. 
\fi

\section{Background and Problem Formulation}\label{sec:Background and Formulation}

\subsection{Learning Objective}
Consider a network of $n$ nodes connected via a base topology $G = (V, E)$, where $V$ is the set of nodes ($|V| = n$) and $E$ is the set of node pairs which can communicate directly. We assume that $G$ is bidirected, i.e., $(i,j)\in E$ if and only if $(j,i)\in E$; we also assume each node to have a self-loop to indicate its access to local information.  
Each node $i \in V$ is associated with a (possibly non-convex) local objective function $f_i(\bm{x})$, which is a function of the parameter $\bm{x} \in \mathds{R}^d$ and the local dataset at node $i$. The objective of DFL is to collaboratively minimize the global objective function
\begin{align}\label{eq:F(x)}
f(\boldsymbol{x}) := \frac{1}{n}\sum_{i=1}^{n} f_{i}(\boldsymbol{x}),
\end{align}
which averages the local objectives across all nodes. 

\textbf{Remark:} Bidirected (or equivalently undirected) base topology is a common assumption in DFL. We inherit this assumption in this work (which allows us to orient activated links arbitrarily as in Alg.~\ref{Alg:Graph Design}), and leave the case of arbitrary directed base topology to future work.

\subsection{Learning Algorithms}\label{subsec:Decentralized Learning Algorithm}

\subsubsection{D-PSGD and its Limitation}

Decentralized learning algorithms typically work by combining local stochastic gradient descent with decentralized approximations of cross-node parameter averaging. In a commonly-used algorithm known as \emph{Decentralized Parallel Stochastic Gradient Descent (D-PSGD)}~\cite{Lian17NIPS}, this combination leads to an iterative algorithm which performs the following update in parallel at every node $i$:  
\looseness=-1
\begin{align}\label{eq:DecenSGD}
    \boldsymbol{x}^{(t+1)}_i = \sum_{j=1}^{n}W^{(t)}_{ij}(\boldsymbol{x}^{(t)}_j - \eta g_j(\boldsymbol{x}^{(t)}_j; \xi^{(t)}_j)),
\end{align}
Here $\bm{x}_i^{(t)}$ denotes the model parameter vector at node $i$ at the start of iteration $t$, $g(\bm{x}_i^{(t)}; \xi_i^{(t)})$ denotes the stochastic gradient computed at that node using a minibatch $\xi_i^{(t)}$ drawn from its local data, $\eta > 0$ denotes the learning rate, and  $\bm{W}^{(t)}=(W^{(t)}_{ij})_{i,j=1}^n$ is the $n \times n$ \emph{mixing matrix} used for parameter aggregation at iteration $t$, which should be topology-compliant (i.e., ${W}^{(t)}_{ij}\neq 0$ only if $(j,i)\in E$). 

The mixing matrix plays a crucial role in the learning performance. On one hand, since node $i$ requires communication from node $j$ in iteration $t$ only if ${W}^{(t)}_{ij} \neq 0$, the mixing matrix controls the communication pattern and hence the communication cost. On the other hand, the mixing matrix also controls the convergence rate. According to \cite{Lian17NIPS}, the mixing matrix should be \emph{symmetric and doubly stochastic}, which simplifies its design. For example, if all iterations use the same deterministic mixing matrix $\bm{W}$, then having a positive spectral gap, i.e.,\footnote{Here $\lambda_i(\bm{W})$ denotes the $i$-th largest eigenvalue of $\bm{W}$.}  
\begin{align}
\rho:= \max(|\lambda_2(\bm{W})|, |\lambda_n(\bm{W})|) < 1,
\end{align}
is sufficient to ensure that the cumulative effect of local averaging weighted by $\bm{W}$ converges to the effect of global averaging with uniform weights, i.e., $\lim_{t\to\infty} \bm{W}^t = {1\over n}\bm{1} \bm{1}^\top$, which leads to a guaranteed convergence for D-PSGD~\cite{Lian17NIPS}. Moreover, the smaller the $\rho$-parameter, the faster the convergence. This observation has inspired a line of works on mixing matrix design for D-PSGD \cite{MATCHA22,Chiu23JSAC,Herrera25OJCS,hua2022efficient,le2023refined}, which aim at optimizing the cost-convergence tradeoff by designing mixing matrices to minimize $\rho$ or equivalent parameters under a budget on the cost per communication round.  

However, the requirement of symmetry significantly restricts the flexibility in mixing matrix design. From a communication perspective, requiring $W^{(t)}_{ij}= W^{(t)}_{ji}$ means that node $i$ needs to send its parameter vector to node $j$ whenever it wants to include $j$'s parameter vector in its local parameter aggregation. Referring to setting $W^{(t)}_{ij}\neq 0$ as \emph{activating (directed) link $(j,i)$}, this means symmetric mixing matrix design can only activate links at the granularity of bidirected pairs $\{(i,j),(j,i)\}$, which limits the solution space.  \looseness=-1

\subsubsection{SGP}

To remove the above limitation, we consider a different learning algorithm called \emph{Stochastic Gradient Push (SGP)}~\cite{Assran19ICML}. The difficulty with asymmetric mixing matrices is that it is often impossible to converge towards global averaging through mixing alone. Fortunately, this problem can be addressed by an algorithm called Push-Sum, which only requires the mixing matrices to be \emph{column-stochastic} (i.e., each column sums to one), but \emph{not necessarily symmetric}. For such matrices, it is still possible to ensure convergence $\lim_{K\to\infty} \prod_{t=0}^K \bm{W}^{(t)} = \bm{\pi}\bm{1}^\top$ for some probability vector $\bm{\pi}$ under mild conditions, but $\pi_i$ may not equal $1/n$ ($\forall i\in V$), causing a bias in the averaging. Push-Sum solves this problem by maintaining an additional scalar parameter $w^{(t)}_i$ at each node $i$, with the initial value $w^{(0)}_i=1$. The update equation changes from \eqref{eq:DecenSGD} to 
\begin{align}
\bm{x}^{(t+1)}_i &= \sum_{j=1}^n W^{(t)}_{ij}\left(\bm{x}^{(t)}_j - \eta g_j\Big({\bm{x}^{(t)}_j\over w^{(t)}_j}; \xi^{(t)}_j\Big)\right), \label{eq:SGP - x}\\
w^{(t+1)}_i &= \sum_{j=1}^n W^{(t)}_{ij} w^{(t)}_j, \label{eq:SGP - w}
\end{align}
where $w^{(t)}_i$ tracks the cumulative bias and $\bm{z}^{(t)}_i:=\bm{x}^{(t)}_i/w^{(t)}_i$ is the \emph{de-biased} parameter used in gradient computation. Compared to \eqref{eq:DecenSGD}, the SGP update only requires the communication of one extra scalar $w^{(t)}_j$ for each $W^{(t)}_{ij}\neq 0$, which incurs negligible cost compared to the communication of parameter vector, but it allows mixing matrix design to activate directed links \emph{individually}, i.e., allowing $W^{(t)}_{ji}=0$ while $W^{(t)}_{ij}\neq 0$.  As shown later, such increased flexibility can lead to notably better cost-convergence tradeoff. 

\subsection{Assumptions for Convergence}\label{subsec:Convergence Assumptions}
Let $\overline{\bm{x}}^{(t)}:={1\over n}\sum_{i=1}^n \bm{x}^{(t)}_i$ denote the learned global model at iteration $t$, i.e., a global average of local models at this iteration. As the objective function $f$ is often non-convex for deep learning, the convergence criterion is based on gradient norm: for any \emph{required level of convergence} $\epsilon>0$, we say that SGP achieves $\epsilon$-convergence if\footnote{In this work, we use $\|\bm{a}\|$ to denote $\ell$-2 norm if $\bm{a}$ is a vector, and spectral norm if $\bm{a}$ is a matrix.}\looseness=-1 
\begin{align} 
\frac{1}{T} \sum_{t=0}^{T-1} \E[\|\nabla f(\boldsymbol{\overline{\bm{x}}}^{(t)})\|^2]\leq \epsilon,
\end{align}
which ensures the learned model to be sufficiently close to a local minimum of the global objective function.  

Given mixing matrices $(\bm{W}^{(t)})_{t=0}^{T-1}$, we use $E^{(t)}:=\{(j,i)\in E:\: W^{(t)}_{ij}\neq 0\}$ to denote the set of \emph{links activated at iteration $t$}, and $\delta$ to denote the \emph{minimum nonzero mixing weight}, defined as
    \begin{align}\label{eq:delta definition}
    \delta := \min_{t}\ \min_{(j,i):\, W^{(t)}_{ij} > 0} W^{(t)}_{ij}.
    \end{align}
The convergence of SGP can be guaranteed under the following assumptions:
\begin{enumerate}
    \item ($L$-smooth) Each local objective function $f_i({\xbf})$ is $L$-Lipschitz smooth, i.e.,  $\norm{\nabla f_i(\xbf) - \nabla f_i(\xbf')} \le L \norm{\xbf - \xbf'}$, $\forall \xbf, \xbf' \in \Rd$.
    \item (Bounded variance) There exists a constant  $\sigma^2$ such that \\
    $\E[\norm{g_i(\xbf; \xi) -\nabla f_i(\xbf)}^2] \leq \sigma^2$, $\forall i$ and $\forall \xbf \in \Rd$. 
    \item (Bounded heterogeneity) There exists a constant $\zeta$ such that
    $\frac{1}{n} \sum_{i=1}^n \norm{\nabla f_i(\xbf) - \nabla f(\xbf) }^2 \leq \zeta^2$, $\forall \xbf \in \Rd$.    
    \item (Mixing connectivity) There exist finite positive integers $B$ and $\Delta$, such that the graph $(V, \bigcup_{t = lB}^{(l+1)B - 1} E^{(t)})$ is strongly connected with diameter at most $\Delta$ for every $l\in \mathds{N}$.
\end{enumerate}
The same assumptions were made in \cite{Assran19ICML,Zhou25TCNS}. While SGP can be extended to guarantee convergence under bounded communication delays~\cite{Assran19ICML}, we assume the synchronized version as define in \eqref{eq:SGP - x}--\eqref{eq:SGP - w} in this work to focus on mixing matrix design.

\subsection{Cost Model}\label{subsec:Cost Models}

We focus on communication time as the cost measure, which usually dominates the computation time in decentralized settings and determines the overall convergence time. Specifically, let $G^{(t)}=(V,E^{(t)})$ denote the \emph{activated communication graph} at iteration $t$ and $\tau(G^{(t)})$ denote the required \emph{number of transmission slots} to complete the activated parameter transmissions. 
We use $\tau(G^{(t)})$ to measure the cost of iteration $t$, assuming one transmission slot as the time to complete a set of parallel parameter transmissions under a feasible communication schedule. 

We assume that each node has a half-duplex omnidirectional transceiver capable of broadcasting to all neighbors or receiving from one neighbor in any slot. Two directed links $(i,j)$ and $(k,l)$ can be scheduled in the same slot if and only if they satisfy \begin{enumerate}
    \item (Half-duplex constraint) $i\neq l$ and $j\neq k$, and
    \item (Interference constraint) $(i,l), (k,j)\not\in E$ if $i\neq k$.
\end{enumerate}
This is consistent with the communication model in \cite{Herrera25OJCS} and \cite{xing2021federated}. Under the above assumption, we can construct a \emph{conflict graph} of the activated links, denoted by $G^{(t)}_c=(E^{(t)}, C^{(t)})$, where $\big((i,j),(k,l)\big)$\\$\in C^{(t)}$ if and only if the links $(i,j)$ and $(k,l)$ cannot be scheduled in the same slot. Then $\tau(G^{(t)})$ is given by the minimum number of vertex‑independent sets in $G^{(t)}_c$ whose union covers all its vertices, i.e., the \emph{chromatic number} 
of the conflict graph.

\textbf{Remark:}
While the above model does not explicitly consider runtime dynamics such as transmission failures and retransmissions caused by noise or external interference, it can incorporate the mean effect of such dynamics by interpreting one transmission slot as the \emph{average time} to successfully complete a set of parallel parameter transmissions under constraints (1--2). Moreover, while the above model implicitly assumes the interference graph to be the same as the connectivity graph (i.e., $G$), this assumption can be easily relaxed by defining the interference constraint according to a separate interference graph, and our solution still applies.

\subsection{Design Objective}\label{subsec:Design Objective}

Generally, mixing matrix design faces an inherent tradeoff between (i) communicating more per iteration in order to converge in fewer iterations and (ii) communicating less per iteration at the cost of converging in more iterations.  To accelerate learning, we aim at designing possibly time-varying mixing matrices $(\bm{W}^{(t)})_{t=0}^{T-1}$ to \emph{minimize the total number of transmission slots}, i.e., 
\begin{align}\label{eq:max per-node energy}
\min_{(\bm{W}^{(t)})_{t=0}^{T-1}} \sum_{t=0}^{T-1} \tau(G^{(t)}),
\end{align}
until SGP reaches $\epsilon$-convergence for a given $\epsilon>0$. \looseness=0

\subsection{Motivating Experiment}\label{subsec:Motivating Experiment}

\begin{figure}[ht]
\vspace{-1em}
    \centering
    \begin{subfigure}[b]{0.23\textwidth}
        \centering
        \includegraphics[height = 1.3in,width=1\textwidth]{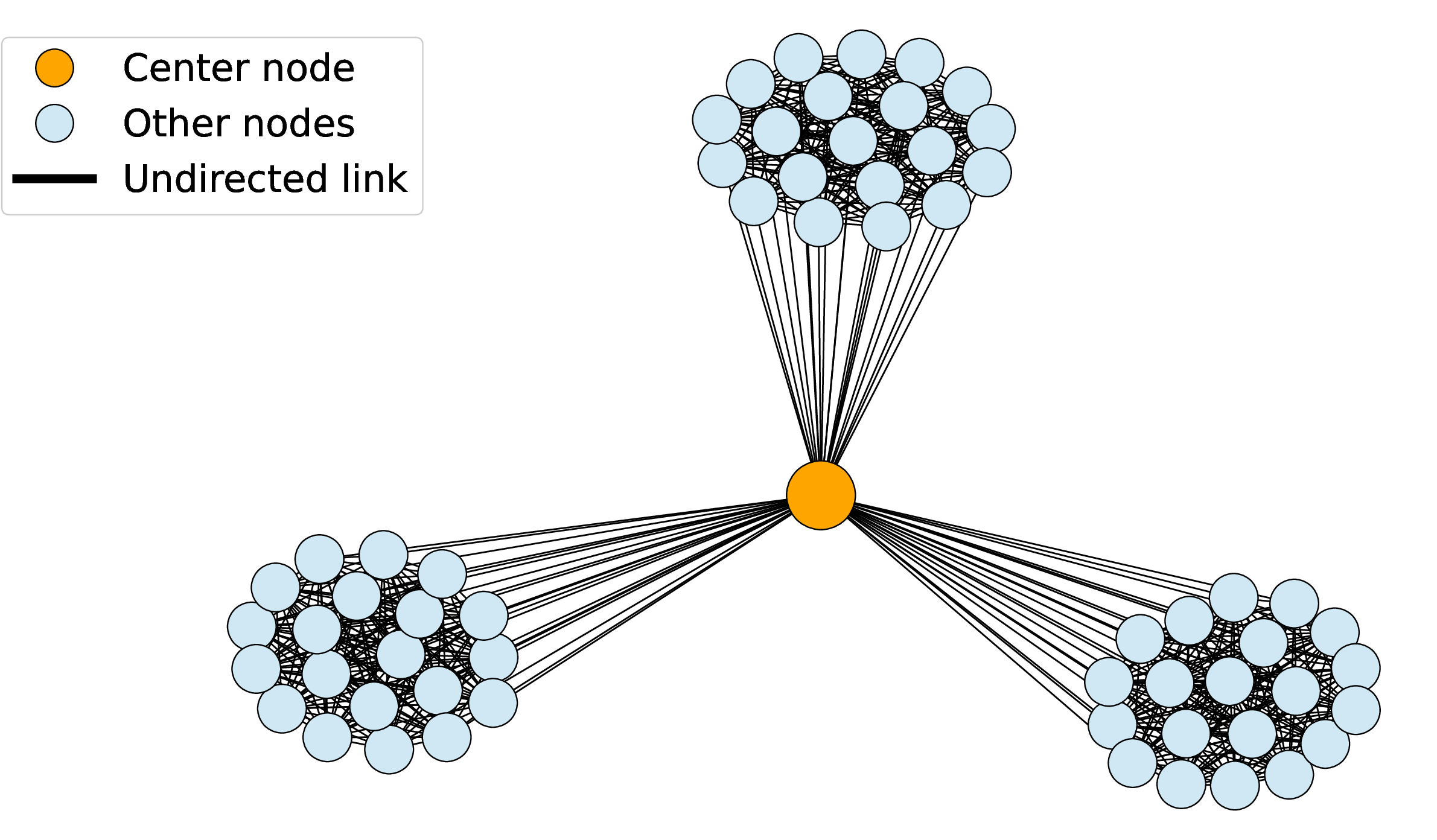}
        \vspace{-2em}
        \caption{Base topology}
        \label{fig:windmill-a}
    \end{subfigure}
    \hfill
    \begin{subfigure}[b]{0.2\textwidth}
        \centering
        \includegraphics[height = 1.3in, width=1.1\textwidth]{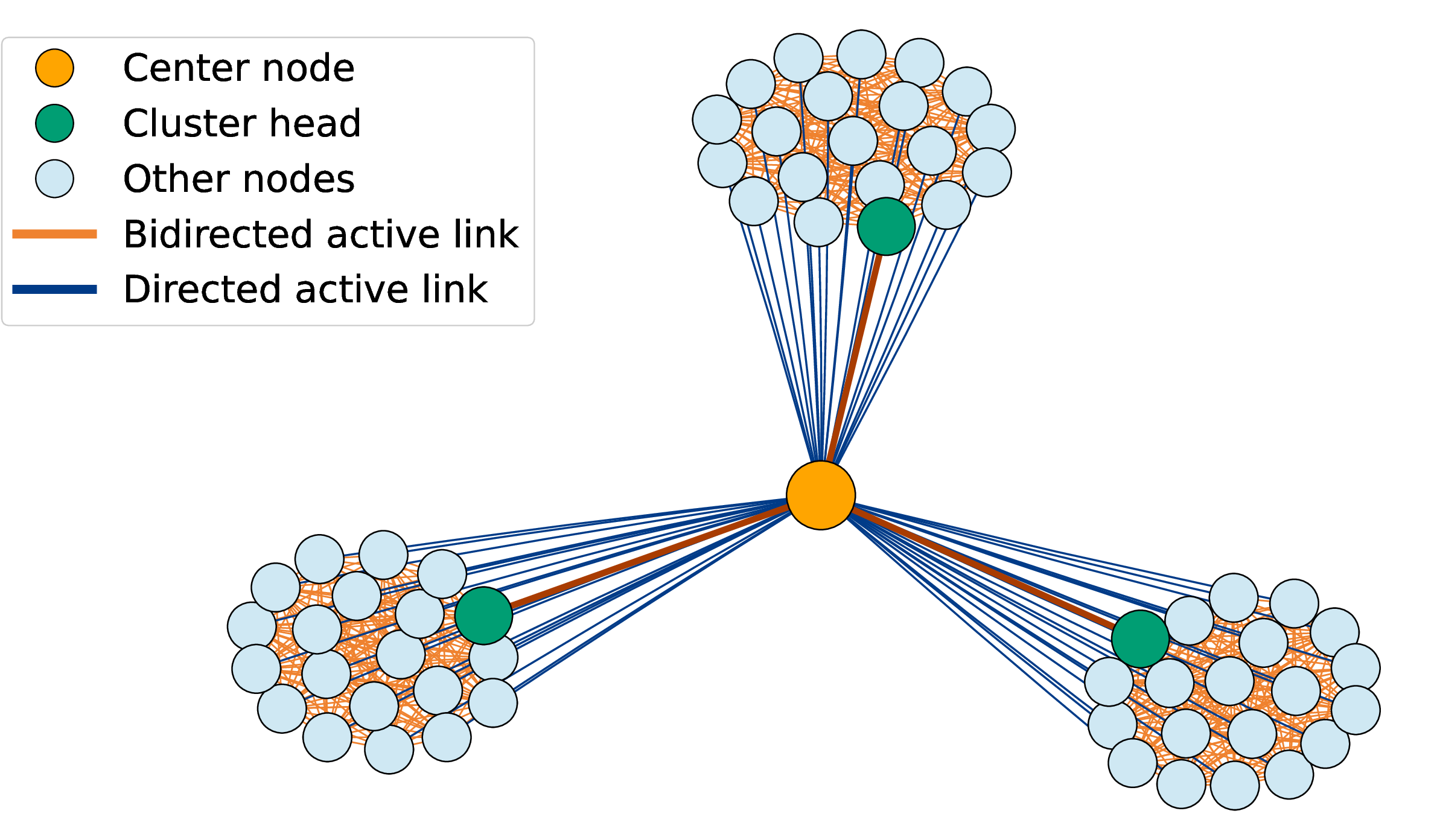}
         \vspace{-2em}
        \caption{Activated graph for SGP}
        \label{fig:windmill-b}
    \end{subfigure}
    \vspace{-1em}
    \caption{Topology for motivating experiment: each undirected link in (a) represents two directed links in opposite directions; (b) contains all the links in (a), except that only one node per cluster (highlighted) communicates to the hub.}
    \label{fig:windmill}
    \vspace{-.75em}
\end{figure}                      

\begin{figure}[ht]
    \centering    \centerline{\mbox{\includegraphics[height = 2.35in,width=1.0\linewidth]{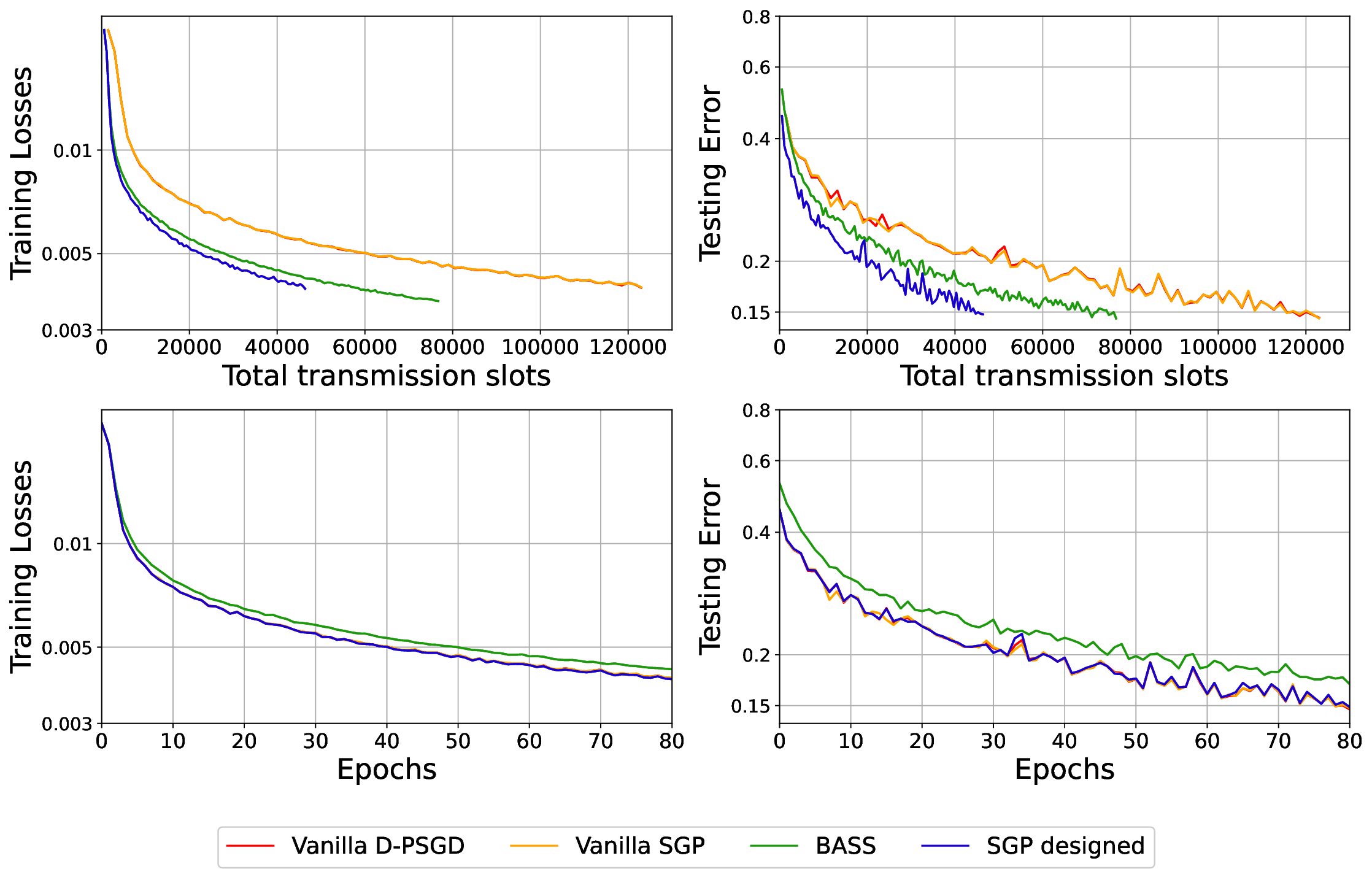}}}
    \vspace{-1em}
     \caption{ Motivating experiment based on FMNIST  over the topology in Fig.~\ref{fig:windmill} (23 slots per iteration for `BASS' and `SGP designed'). 
    }
    \label{fig:results_fmnist_windmill}
    \vspace{-1em}
\end{figure}


As a motivating example, we compare the convergence performance of state-of-the-art solutions and SGP based on a designed mixing matrix. Our experiment is based on the FMNIST dataset~\cite{FMNIST}, randomly distributed across the 61 nodes in a (3,21)-windmill graph  shown in Fig.~\ref{fig:windmill-a}, which contains three 21-node cliques sharing a common node. For example, this can model a wireless network with partitioned dense clusters bridged by a hub with line-of-sight links to all the other nodes. The learning task is to train a 4-layer CNN with 1,663,370 parameters as in \cite{McMahan17AISTATS}, with a batch size of 64 and a learning rate of $0.02$. 
We evaluate four solutions: (1) `Vanilla D-PSGD' using the base topology as the communication graph (i.e., all neighbors communicate) and Metropolis-Hastings weights; (2) `Vanilla SGP' using the base topology as the communication graph and column-wise uniform weights (i.e., equally spitting the unit weight of each column among the out-neighbors); (3) `BASS'~\cite{Herrera25OJCS}, a state-of-the-art mixing matrix design for D-PSGD under broadcast communications (based on `heuristic design'); (4) `SGP designed', which is SGP based on the communication graph in Fig.~\ref{fig:windmill-b} and column-wise uniform weights. This design activates all the directed links in the base topology, except for the incoming links to the hub from all but one node in each cluster (which functions as a ``cluster head'' to share information with other clusters through the hub). This is the design produced by our proposed algorithm (Alg.~\ref{Alg:Graph Design}); see Fig.~\ref{fig:example}e for the detailed explanations based on a simpler topology of the same type.   
The results, shown in Fig.~\ref{fig:results_fmnist_windmill}, suggest that: (i) while communicating according to the base topology as in `Vanilla D-PSGD/SGP' is good for maximizing the convergence rate measured in epochs, using a sparser communication graph like `BASS' or `SGP designed' can achieve faster convergence in  wall-clock time (measured in slots), and (ii) suitably utilizing the asymmetric parameter sharing capability of SGP can substantially reduce the convergence time compared to the state-of-the-art designs based on D-PSGD (specifically, `SGP designed' reduces the convergence time by {40\%} compared to `BASS' and {62\%} compared to `Vanilla D-PSGD/SGP' in achieving $85\%$ testing accuracy).  \looseness=-1


\section{Convergence Analysis}\label{sec:Convergence Analysis}

While the convergence of SGP has been analyzed in \cite{Assran19ICML,Spiridonoff20JMLR, Zhou25TCNS}, the existing convergence bounds have multi-term, highly complicated dependencies on the mixing matrices, which fail to provide a tractable objective for mixing matrix design. To fill this gap, we derive a new convergence bound for SGP based on the assumptions in Section~\ref{subsec:Convergence Assumptions} as follows.

\begin{theorem}\label{thm:SGP convergence}
    Let \(A := 2f(\xbar^{(0)}) - 2f^* + L\sigma^2\) and \( S := (\max_{i\in V} \norm{\bm{x}_i^{(0)}})^2 \)\\\(
    + n^2\sigma^2 + 3n^2\zeta^2.\) SGP with $\eta=\sqrt{n/T}$ under assumptions (1)--(4) in Section~\ref{subsec:Convergence Assumptions} achieves $\epsilon$-convergence (i.e., $\frac{1}{T} \sum_{t=0}^{T-1} \EE{ \norm{\nabla f(\overline{\xbf}^{(t)})}^2} \le \epsilon$) if the number of iterations $T$ satisfies T \( \geq \frac{24L^2C^2}{(1-q)^2\epsilon}S \)  for any \(\epsilon\) satisfying \looseness=-1
    \begin{align}\label{eq:condition on epsilon}
\hspace{-2em}&\frac{2(1-q)^2A^2}{3L^2C^2nS}
\le
\epsilon
\le
\min\left\{
\frac{4S}{3n^2},\;
\frac{24L^2C^2S}{n(1-q)^2}
\right\},
\end{align}
where \(C:=\frac{4}{\delta^{\Delta B}}\) and \(q:=(1-\delta^{\Delta B})^{\frac{1}{\Delta B}}\) are the only parameters depending on mixing matrices,  $\bm{x}^{(0)}_i$ is the initial parameter vector at node $i$, and $f^*$ is the optimal objective value. 
\end{theorem}

\begin{corollary}\label{cor:SGP convergence}
For $\delta^{\Delta B}\ll 1$ and $\epsilon$ satisfying \eqref{eq:condition on epsilon}, SGP under assumptions (1)--(4) in Section~\ref{subsec:Convergence Assumptions} achieves $\epsilon$-convergence when the number of iterations $T$ reaches
 \begin{align}
\underline{T} := \frac{24L^2C^2}{(1-q)^2\epsilon}S= O\left(\frac{\Delta^2 B^2}{\delta^{4\Delta B}}\right), \label{eq:T bound}
 \end{align}
 where the big-$O$ notation hides input parameters independent of the mixing matrices. 
\end{corollary}

\textbf{Remark:}
In large-scale learning, the minimum mixing weight $\delta$ is usually small and the diameter $\Delta$ of the activated communication graph is usually large, implying that $\delta^{\Delta B}\ll 1$, $C\gg 1$, and $1-q\ll 1$. Thus, the condition \eqref{eq:condition on epsilon} holds easily. In this case, Corollary~\ref{cor:SGP convergence} states that the number of iterations till convergence only depends on mixing matrices through: the \emph{number of iterations $B$} to reach strong connectivity, the \emph{diameter $\Delta$} of the communication graph over $B$ iterations, and the \emph{minimum mixing weight $\delta$}, in the form of \eqref{eq:T bound}.

\section{Simplification of Design Objective}\label{sec:Design Objective}

The clean bound in Corollary~\ref{cor:SGP convergence} provides us an explicit objective function for mixing matrix design. Considering periodic mixing matrix design, where $\bm{W}^{(lB + t)} = \bm{W}^{(t)}$ for all $l\in \mathds{N}$ and $t\in \{0,\ldots,B-1\}$, our goal is to design $B$ mixing matrices $(\bm{W}^{(t)})_{t=0}^{B-1}$ with the corresponding activated graphs $(G^{(t)})_{t=0}^{B-1}$, such that the total number of transmission slots, given by 
\begin{align}
{\underline{T}\over B}  \sum_{t=0}^{B-1}\tau(G^{(t)}) \propto
\frac{\Delta^2 B}{\delta^{4\Delta B}}\sum_{t=0}^{B-1}\tau(G^{(t)}), \label{eq:objective - 1}
\end{align}
is minimized. 
However, directly optimizing \eqref{eq:objective - 1} is intractable as it mixes topology design with weight design and involves an unknown number of topologies. Below, we will simplify this objective into a more tractable objective function to facilitate a solution. 

\subsection{Optimal Weight Assignment}\label{subsec:Uniform Weight}

As the objective \eqref{eq:objective - 1} depends on mixing weights only through the minimum weight $\delta$, it is easy to show that \emph{uniform weight assignment}, i.e., equally splitting the unit weight of each column among the activated out-neighbors, is optimal. 

Specifically, given an activated graph $G^{(t)}$, we define the activated out-neighborhood of node $j$ (excluding the self-loop) as:
    \[
    N_t^{+}(j) := \{\, i \in V\setminus\{j\} : (j, i)\in E^{(t)}\,\},      
    \]
and the activated out-degree as $d_t^{+}(j) := |N_t^{+}(j)|$. We have the following observation.

\begin{lemma}\label{lem: uniform weight}
Under given activated graphs $(G^{(t)})_{t=0}^{B-1}$ and column-stochastic mixing, the conditionally optimal value of $\delta$ is 
\begin{align}\label{eq: uniform weight bound}
   \max \; \delta =  \min_{t\in\{0,\dots,B-1\}} \ \min_{j\in V} \frac{1}{d_t^{+}(j)+1},
\end{align}
achieved when $W^{(t)}_{ij}=W^{(t)}_{jj}=1/ (d_t^{+}(j)+1)$ for all $j\in V$ and $(j,i) \in E^{(t)}$. 
\end{lemma}

\textbf{Remark:} Lemma~\ref{lem: uniform weight} implies that once the activated graphs are given, it suffices to simply assign mixing weights to the activated links uniformly as in Lemma~\ref{lem: uniform weight}.

\subsection{Optimal Period Length}

We will further show that instead of considering all the possible periods, it suffices to consider the special case of $B=1$, i.e., all the iterations use the same activated graph for parameter sharing. 

Let $G_a:= (V, \bigcup_{t=0}^{B-1}E^{(t)})$ denote the \emph{per-period activated graph}. We will use the following graph-theoretic notions in our derivation:
\begin{itemize}
    \item $d_a^+(j)$ and $d_a^-(j)$ denote the out/in-degree of node $j$ in $G_a$;
    \item $D_a^+$ and $D_a^-$ denote the maximum out/in-degree of $G_a$. 
\end{itemize}  

First, we show that for a given $G_a$ and a given $B\geq 1$, splitting the outgoing links in $G_a$ at the node with the maximum out-degree evenly among the $B$ iterations helps to optimize the minimum weight $\delta$. 

\begin{lemma}\label{lem: equally split}
    Given the period $B$ and the per-period activated graph $G_a$, we have that
     \begin{align}\label{eq: max tau}
            \max \; \delta \leq  \frac{1}{\left\lceil D^+_a/B\right\rceil+1}.
    \end{align}       
\end{lemma}


Plugging the bound in \eqref{eq: max tau} into \eqref{eq:objective - 1} yields a lower bound on the objective in \eqref{eq:objective - 1} under period length $B$, denoted by 
\begin{align}\label{eq: obj equal}
F(B):= \Big(\sum_{t=0}^{B-1}\tau(G^{(t)})\Big)\Delta^2 B
\left(\left\lceil \frac{D^+_a}{B}\right\rceil + 1\right)^{4\Delta B}.
\end{align}

Next, we show that this lower bound is minimized at $B=1$. 




\begin{lemma}\label{lem: B=1}
 Given a graph $G_a=(V, E_a)$ to be activated over a $B$-iteration period, $F(B)$ defined in \eqref{eq: obj equal} satisfies
 \begin{align}
F(B) \geq F(1) = \tau(G_a)\Delta^2 (1+D_a^+)^{4\Delta} \label{eq:F(1)}
 \end{align}
for any integer $B\geq 1$ and any $G^{(t)}=(V, E^{(t)})$ ($t\in\{0,\ldots,B-1\}$) such that $\bigcup_{t=0}^{B-1}E^{(t)}=E_a$.
\end{lemma}

Given a per-period activated graph $G_a=(V,E_a)$, activating all the links in $G_a$ in the same iteration and assigning weights uniformly achieve the lower bound on the objective function \eqref{eq:objective - 1} given by \eqref{eq:F(1)}. This observation greatly simplifies our mixing matrix design problem from designing $B$ $n\times n$ matrices to designing a single graph $G_a$ that connects all the nodes in $V$, with the objective of\looseness=-1
\begin{align}\label{eq: clean obj}
    \min_{G_a} \;\tau(G_a) \Delta^2(1+D^+_a)^{4\Delta}, 
\end{align}
after which we can construct a mixing matrix $\bm{W}$ via uniform weight assignment as in Section~\ref{subsec:Uniform Weight} to use in each iteration.

\subsection{Computable Cost Bound}\label{subsec:Cost Bounds}

The main difficulty in optimizing the new objective \eqref{eq: clean obj} is that the per-iteration communication cost $\tau(G_a)$ is not an explicit function of $G_a$ and is difficult to evaluate. In fact, for the cost model considered in Section~\ref{subsec:Cost Models}, even evaluating $\tau(G_a)$ for a given $G_a$ is NP-hard: 
by definition, $\tau(G_a)$ is the chromatic number of the conflict graph of $G_a$, which is NP-hard to compute~\cite{GareyJohnson1979}. 
This motivates us to relax $\tau(G_a)$ into a computable bound that explicitly depends on the structure of $G_a$. 

Let $G_c=(E_a, C_a)$ denote the \emph{conflict graph} of $G_a$, where each vertex denotes an activated link and each edge \(\big((i,j),(k,l) \big) \in C_a\) denotes a scheduling conflict\footnote{We use ``edge'' to refer to an undirected link in the graph-theoretic sense.}. Let $\chi(G_c)$ denote the \emph{chromatic number of $G_c$}, which by definition equals $\tau(G_a)$.  
We will use a few more graph-theoretic notions in the following derivation:
\begin{itemize}
    \item $D_c$ denotes the maximum degree of the conflict graph $G_c$ (note that $G_c$ is undirected);
    \item $D$ denotes the maximum in/out-degree of the base topology $G$ (note that the maximum in-degree equals the maximum out-degree as $G$ is bidirected).
\end{itemize}  

Base on Brooks' Theorem \cite{Diestel2017GraphTheory}, the conflict graph $G_c$ can be colored with at most $D_c$ colors except for two cases, complete graphs and cycle graphs of odd length, which require  $D_c+1$ colors. 
Hence, we always have 
\begin{align}
\tau(G_a) = \chi(G_c) \le D_c+1. \label{eq:tau <= Dc+1}
\end{align}
What remains is to connect $D_c$ to the structure of $G_a$, which will depend on the specific communication model.


Under the broadcast communication model assumed in Section~\ref{subsec:Cost Models}, two links $(i,j)$ and $(k,l)$ conflict if and only if they violate the half-duplex constraint or the interference constraint, which leads to the following bound on $D_c$.

\begin{lemma}\label{lem:Dc bound - broadcast}
Under the communication model assumed in Section~\ref{subsec:Cost Models}, the maximum degree of the conflict graph is bounded by\looseness=-1
\begin{align}
D_c \le\;
(D+1) \bigl(D_a^+ + D_a^-\bigr).\label{eq: Dc bound - broadcast}
\end{align}
\end{lemma}

\subsection{Closed-form Design Objective}

Combining \eqref{eq:tau <= Dc+1} and \eqref{eq: Dc bound - broadcast} yields a per-iteration cost bound of
\[\tau(G_a)\leq (D+1)(D_a^+ + D_a^-)+1 \leq 2(D+1)(D_a^+ + D_a^-).\]
Replacing $\tau(G_a)$ in \eqref{eq: clean obj} by this upper bound leads to a new design objective: 
\begin{align} 
\min_{G_a} \quad  \bigl(D_a^+ + D_a^-\bigr) \Delta^2 (1+D_a^+)^{4\Delta}, \label{eq: G_a obj}
\end{align}
which is an explicit, easily computable function of the activated communication graph $G_a$ that only depends on $G_a$ through three of its graph-theoretic parameters: the maximum in-degree $D_a^-$, the maximum out-degree $D_a^+$, and the diameter $\Delta$. 

\section{Communication Graph Design}\label{sec:Algorithm}

The simplified objective \eqref{eq: G_a obj} reduces our mixing matrix design problem to a graph-theoretic problem of constructing a directed subgraph of $G$ to minimize \eqref{eq: G_a obj}, subject to the strong connectivity constraint imposed by assumption~(4) in Section~\ref{subsec:Convergence Assumptions}. 

\subsection{Design Algorithm}\label{subsec:Design Algorithm}

While our problem has an application-specific objective function \eqref{eq: G_a obj} different from the objectives of classical graph algorithms, the need of optimizing the graph-theoretic parameters $D_a^+$, $D_a^-$, and $\Delta$ motivates a 4-step algorithm as shown in Alg.~\ref{Alg:Graph Design}.  
We will leverage the following notions from graph theory. We use ``edge'' to refer to an undirected link and ``link'' to refer to a directed link. 

\begin{definition}
Given an undirected graph $H=(V(H), E(H))$,  
\begin{itemize}
    \item a \emph{minimum-degree spanning tree} $T$ of $H$ is a spanning tree of $H$ that minimizes the maximum node degree;
    \item a \emph{bridge} (a.k.a. cut edge) of $H$ is an edge in $E(H)$ whose removal will increase the number of connected components, and a \emph{bridge-connected component} of $H$ is a maximal subgraph not containing any bridge; 
    \item the \emph{preorder number $pre(v)$} of a vertex $v$ in $H$ is the order in $\{1,\ldots,|V(H)|\}$ that $v$ is first visited by a depth-first search (DFS) of $H$ from some root vertex.
\end{itemize}

\end{definition}

\textbf{Step~1: Construct a minimum-degree spanning tree.}
Motivated by the need to achieve connectivity with the minimum degrees, we start with the minimum-degree spanning tree of the base topology. However, since existing algorithms for constructing (approximate) minimum-degree spanning trees are designed for undirected graphs, we initially construct a spanning tree for the undirected version of the base topology (denoted by $G_u$), which will be converted into a directed graph later (Step~3).  

\textbf{Step~2: Add $K$ extra edges based on distance.}
Motivated by the need to minimize the diameter (i.e., the maximum distance between any two nodes), we add extra edges to connect nodes with the maximum distance, under connectivity constraints specified by the base topology. Since adding edges may increase the maximum degree, we use a design parameter $K$ to control this tradeoff, which can be optimized according to our design objective \eqref{eq: clean obj} or \eqref{eq: G_a obj}. 

\textbf{Step~3: Orient the edges.}
Using the connected undirected graph from Steps~1--2 as a basis, we construct a directed communication graph $G_a$ by orienting the edges.  Edges inside each bridge-connected component are oriented according to preorder numbers (lines~\ref{Graph Design:12}--\ref{Graph Design:16}): edges in the DFS tree are oriented in ascending preorder numbers (parent-to-child edges); edges not in the DFS tree are oriented in descending preorder numbers (back edges). We will show that such orientation turns each bridge-connected component into a strongly connected subgraph. Then we connect these components with bidirected links (line~\ref{Graph Design:17}) to form a strongly connected graph $G_a$. 

\textbf{Step~4: Cost-preserving link augmentation.}
It is possible that there are directed links not selected in Steps~1--3 that can also be activated without increasing the per-iteration communication cost. That is, given a feasible communication schedule for the $G_a$ constructed so far as a slot assignment $\mathcal{S}=\{S_1,\ldots,S_{\tau}\}$ (where each $S_t$ is a subset of the activated links that can be scheduled in the same slot without conflict), there may be other links that can be added to the schedule without causing conflict or increasing the schedule length. Intuitively, adding such links to $G_a$ will further accelerate convergence without costing more time per iteration. Thus, we augment $G_a$ with such non-conflicting links, prioritizing the link that minimizes the factor $\Delta^2 (1+D_a^+)^{4\Delta}$ in our design objective  that represents the iteration count (line~\ref{Graph Design:22}). However, since adding links to $G_a$ may degrade $\Delta^2 (1+D_a^+)^{4\Delta}$ due to the potential increase in $D_a^+$, we record the objective value achieved by Step~3 (line~\ref{Graph Design:record gamma}) and only add a non-conflicting link if it will not make the objective value worse (lines~\ref{Graph Design:begin If}--\ref{Graph Design:end If}).

\begin{algorithm}[tb]
\small
\SetKwInOut{Input}{input}\SetKwInOut{Output}{output}
\Input{A bidirected base topology $G=(V,E)$ and its undirected version $G_u=(V,E_u)$, \#added edges $K$}
\Output{A strongly connected communication graph $G_a$ for SGP}

\BlankLine
\textbf{Step 1: Construct a minimum-degree spanning tree}\;
\Indp
Compute a (possibly approximate) minimum-degree spanning tree $T$ of $G_u$, with edges $E(T)$\;
\Indm

\BlankLine
\textbf{Step 2: Add $K$ extra edges based on distance}\;
\Indp
\For{$k=1$ \KwTo $K$}{
   Select $(u^\star,v^\star)\in E_u\setminus E(T)$ with the maximum distance between $u^\star$ and $v^\star$ in $T$\;
   Add edge $(u^\star,v^\star)$ to $T$\;
}
\Indm

\BlankLine
\textbf{Step 3: Orient the edges}\; 
\Indp
Decompose $T$ into bridge-connected components, denoted by $\mathcal{C}$, connected by bridges $E_b$\; \label{Graph Design:8}
Initialize an empty digraph $G_a$ on the vertex set $V$\;
\ForEach{bridge-connected component $C \in \mathcal{C}$}{
    Perform DFS on $C$ from any vertex to compute a DFS tree $T_C$ and a preorder number $pre(\cdot)$ for each vertex in $C$\;
\ForEach{edge $(u,v)$ in $C$ with $pre(u)<pre(v)$ \label{Graph Design:12}}
{\If{$(u,v)$ is in $T_C$}
{Add directed link $(u,v)$ to $G_a$\;} 
\Else 
{Add directed link $(v,u)$ to $G_a$\;}
}\label{Graph Design:16}  
}
Add the directed links in both directions for each bridge in $E_b$ to $G_a$\; \label{Graph Design:17}
\Indm

\BlankLine
\textbf{Step 4: Cost-preserving link augmentation}\;
\Indp
Compute a feasible communication schedule for $G_a$ in terms of slot assignment $\mathcal{S}=\{S_1,\dots,S_\tau\}$\;
Compute $\gamma\leftarrow \Delta^2 (1+D_a^+)^{4\Delta}$\; \label{Graph Design:record gamma}
\While{$\exists$a candidate link $e$ in $E\setminus E(G_a)$ that can be added to the schedule $\mathcal{S}$ without conflict \label{Graph Design:20}}{
Select the non-conflicting candidate link $e^*$ that will minimize $\Delta^2 (1+D_a^+)^{4\Delta}$ if added to $G_a$\; \label{Graph Design:22}
\If{adding $e^*$ to $G_a$ will make $\Delta^2 (1+D_a^+)^{4\Delta} \leq \gamma$\label{Graph Design:begin If}}
{Add $e^*$ to $G_a$ and assign it to a non-conflicting slot\; \label{Graph Design:23}}
\Else{Break\;}\label{Graph Design:end If}
}
\Indm
\Return{$G_a$}\;
\caption{Communication Graph Design for SGP 
}
\label{Alg:Graph Design}
\vspace{-.05em}
\end{algorithm}



\subsection{Algorithm Analysis}

\subsubsection{Performance Analysis}

Alg.~\ref{Alg:Graph Design} is guaranteed to provide a feasible solution with guaranteed performance in terms of our design objective, as stated below. 

\begin{theorem}\label{thm:graph design guarantee}
Let 
$D^*$/$\Delta^*$ denote the maximum degree/diameter of the minimum-degree spanning tree of the undirected version of the base topology. Suppose that the optimal minimum-degree spanning tree is computed in Step~1. 
Then under an input parameter $K$ that minimizes \eqref{eq: G_a obj}, a simplified version of Alg.~\ref{Alg:Graph Design} that skips Step~4 returns a strongly connected graph $G_a$, for which 
\begin{align}
\mbox{the value of \eqref{eq: G_a obj}} \leq 2D^* (\Delta^*)^2  (1+D^*)^{4\Delta^*}.  \label{eq:obj bound}
\end{align}
\end{theorem}

\textbf{Remark~1:} Since finding the optimal minimum-degree spanning tree is NP-hard~\cite{Furer92SODA}, in practice we can only compute an approximation to the minimum-degree spanning tree in Step~1. In this case, the performance guarantee in Theorem~\ref{thm:graph design guarantee} still holds, except that $D^*$ and $\Delta^*$ denote the maximum degree and diameter of the spanning tree constructed in Step~1. 

\textbf{Remark~2:} While the link augmentation in Step~4 may degrade our simplified design objective \eqref{eq: G_a obj} due to the potential increase in $D_a^+$ and $D_a^-$, it is guaranteed to maintain or improve the original design objective \eqref{eq: clean obj}. More importantly, our ablation study confirms that this step can significantly improve the actual convergence rate; 
\if\thisismainpaper 1
see Appendix~A.3 in \cite{Nguyen26arXiv_SGP}. 
\else
see Appendix~\ref{appendix:Additional Evaluation}. 
\fi

\subsubsection{Complexity Analysis}

Let $n = |V|$ and $m = |E|$ for the base topology. Step~1 runs in time $O(nm^2)$ using the heuristic minimum-degree spanning tree algorithm from \cite{Furer92SODA}. Step~2 can be completed in time $O\bigl(Kn(n\log{n}+m)\bigr)$ by computing all-pair shortest paths for each $k$ before selecting $(u^\star,v^\star)$. Step~3 runs in $O(n+m)$ because all its operations run in linear time. In Step~4, computing a feasible communication schedule can be completed in $O(m^2)$ time as the conflict graph construction and the vertex coloring  of the conflict graph (by a greedy algorithm) both take $O(m^2)$ time. During each loop in line~\ref{Graph Design:20},  each of the $O(m)$ candidate links is checked against $O(m)$ activated links for conflict, and if a non-conflicting slot is found, we need to recompute $\Delta$ and $D^+_a$ after adding this link, which takes $O(n(m+n\log{n}))$ time. This means each loop can complete in $O(mn(m+n\log{n}))$ time. As at most $m$ links can be added, the running time for Step~4 is $O(m^2n(m+n\log{n}))$. 
The overall time complexity of Alg.~\ref{Alg:Graph Design} is 
$O(m^2 n(m+n\log{n}))$, dominated by the \textbf{while} loop in Step~4. Thus, the complexity of Alg.~\ref{Alg:Graph Design} is polynomial in the size of the base topology.

\subsection{Illustrative Example}
Fig.~\ref{fig:example} illustrates the main steps of Alg.~\ref{Alg:Graph Design} on a \((2,6)\)-windmill graph. 
Starting from the base topology in Fig.~\ref{fig:ex-a}, the algorithm first constructs a spanning tree that  minimizes the maximum node degree, as shown in Fig.~\ref{fig:ex-b}. 
Next, in Fig.~\ref{fig:ex-c}, we add \(K=2\) edges to reduce the diameter while preserving a sparse structure. 
The resulting undirected graph is then oriented in Fig.~\ref{fig:ex-d} to produce a strongly connected directed graph. 
Finally, Fig.~\ref{fig:ex-e} augments this directed graph with non-conflicting links while preserving the communication cost in terms of the required number of transmission slots, which yields the designed communication graph \(G_a\). 

\begin{figure}[ht]
\vspace{-.5em}
    \centering
    \begin{subfigure}[b]{0.25\textwidth}
    \centering
      \includegraphics[height=1.1in,width=1.5\textwidth]{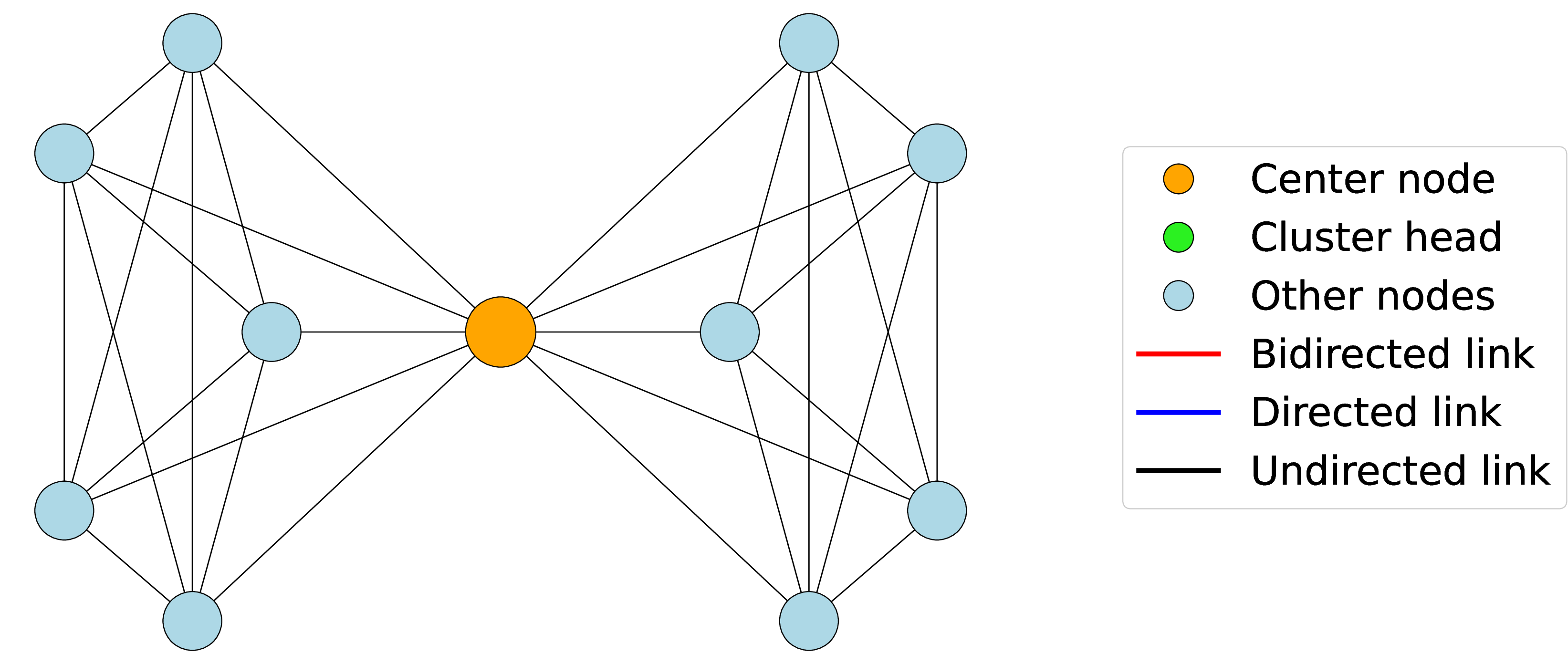}
        \vspace{-1.8em}
        \caption{Base topology}
        \label{fig:ex-a}
    \end{subfigure}
    \hfill
    \begin{subfigure}[b]{0.23\textwidth}
        \centering
        \includegraphics[width=\textwidth]{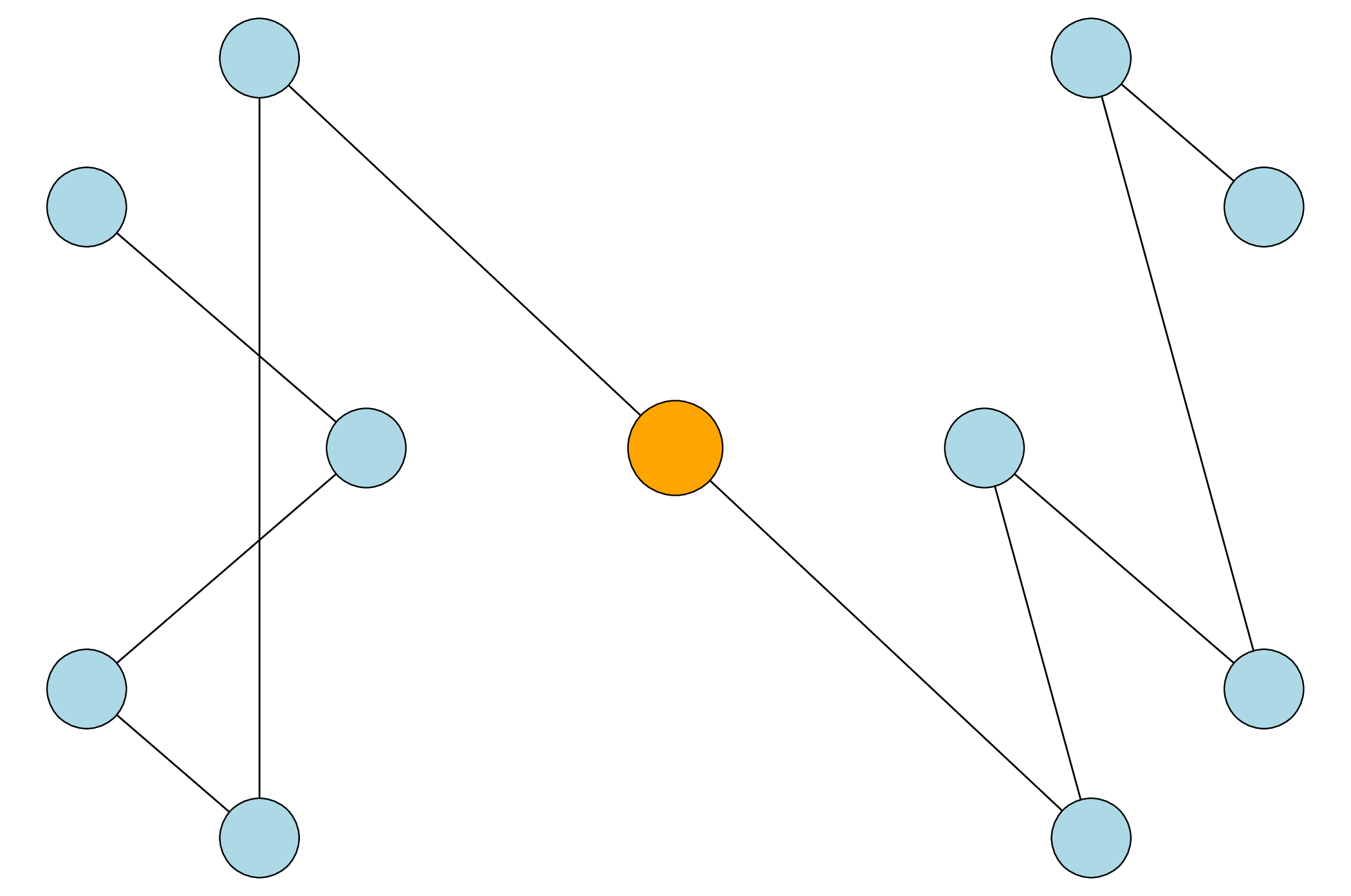}
         \vspace{-2em}
        \caption{Step 1}
        \label{fig:ex-b}
    \end{subfigure}
    \hfill
    \begin{subfigure}[b]{0.23\textwidth}
        \centering
        \includegraphics[width=\textwidth]{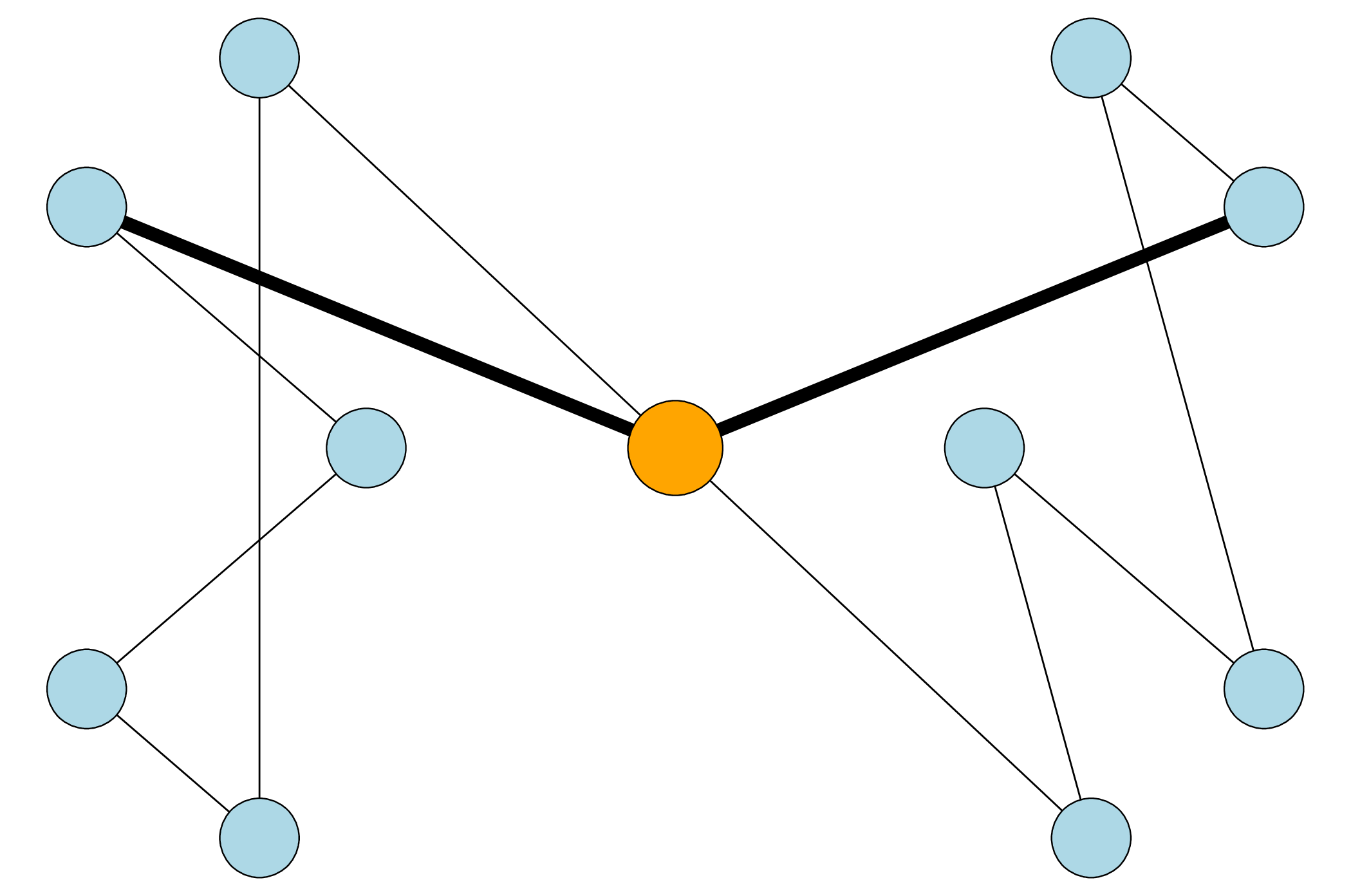}
         \vspace{-2em}
        \caption{Step 2: K=2}
        \label{fig:ex-c}
    \end{subfigure}
    \hfill
    \begin{subfigure}[b]{0.23\textwidth}
        \centering
        \includegraphics[width=\textwidth]{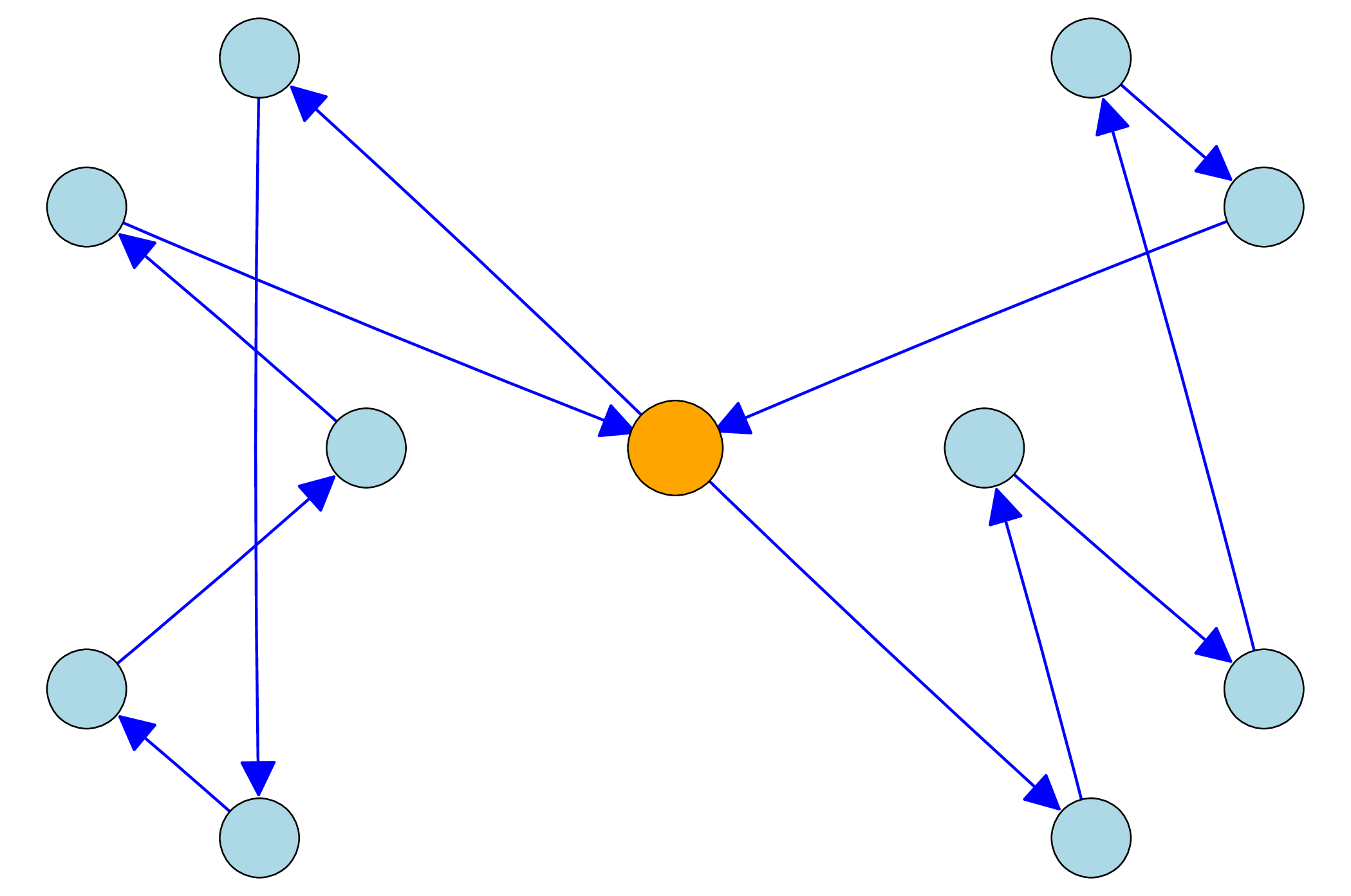}
         \vspace{-2em}
        \caption{Step 3}
        \label{fig:ex-d}
    \end{subfigure}
    \hfill
    \begin{subfigure}[b]{0.23\textwidth}
        \centering
        \includegraphics[width=\textwidth]{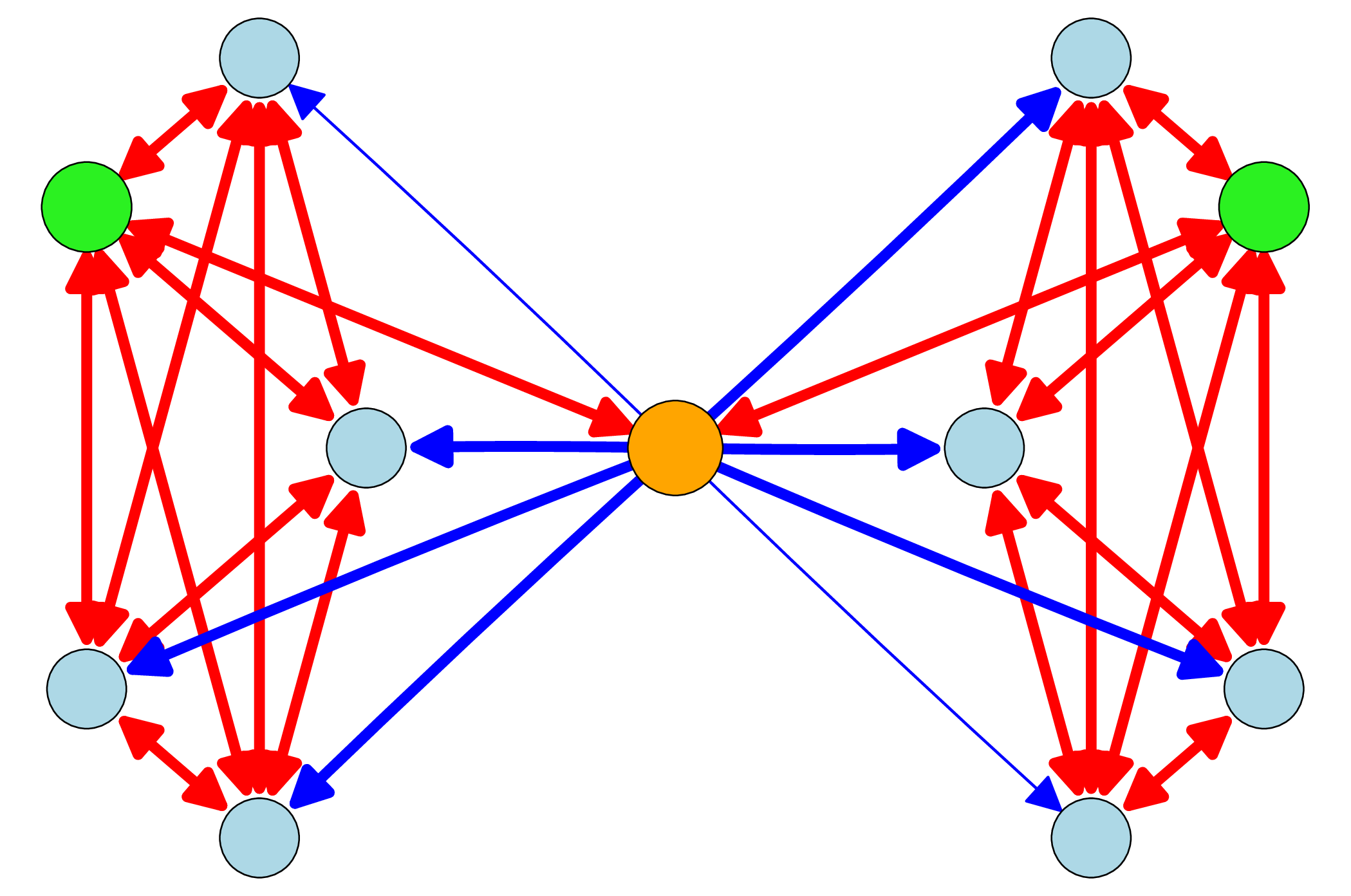}
         \vspace{-2em}
        \caption{Step 4}
        \label{fig:ex-e}
    \end{subfigure}
    \vspace{-1em}
    \caption{Illustration of Alg.~\ref{Alg:Graph Design} based on (2,6)-windmill graph.}
    \label{fig:example}
    \vspace{-2em}   
\end{figure}

\section{Performance Evaluation}\label{sec:Performance Evaluation}

We evaluate the proposed solution against benchmarks on a real dataset under realistic settings. 

\subsection{Evaluation Setting}

\subsubsection{Problem Setting}
We consider the standard task of image classification based on CIFAR-10, which consists of 60,000 color images in 10 classes. We train a lightweight version of ResNet-50 with {1.5M} parameters over its training dataset with 50,000 images, and then test the trained model on the testing dataset with 10,000 images. We set the learning rate to 0.01, and the batch size to 64. {We stop training when the average testing accuracy over 5 consecutive epochs reaches 80\%.} 
We randomly distribute the training data over nodes in two base topologies: (i) a random geometric (RG) graph as shown in {Fig.~\ref{fig: base random}}, generated by uniformly distributing 33 nodes in a unit area and forming links according to a communication radius of 0.5 (resulting in 267 undirected links), which models a free‑space wireless network with identical nodes; (ii) the topology of Roofnet~\cite{Roofnet} as shown in {Fig.~\ref{fig: base roofnet}}, which is a wireless mesh network with 33 nodes and 187 undirected links. 
 \begin{figure}[ht]
\vspace{-1em}
    \centering
    \begin{subfigure}[b]{0.24\textwidth}
        \centering
        \includegraphics[height = 1.2in,width=1\textwidth]{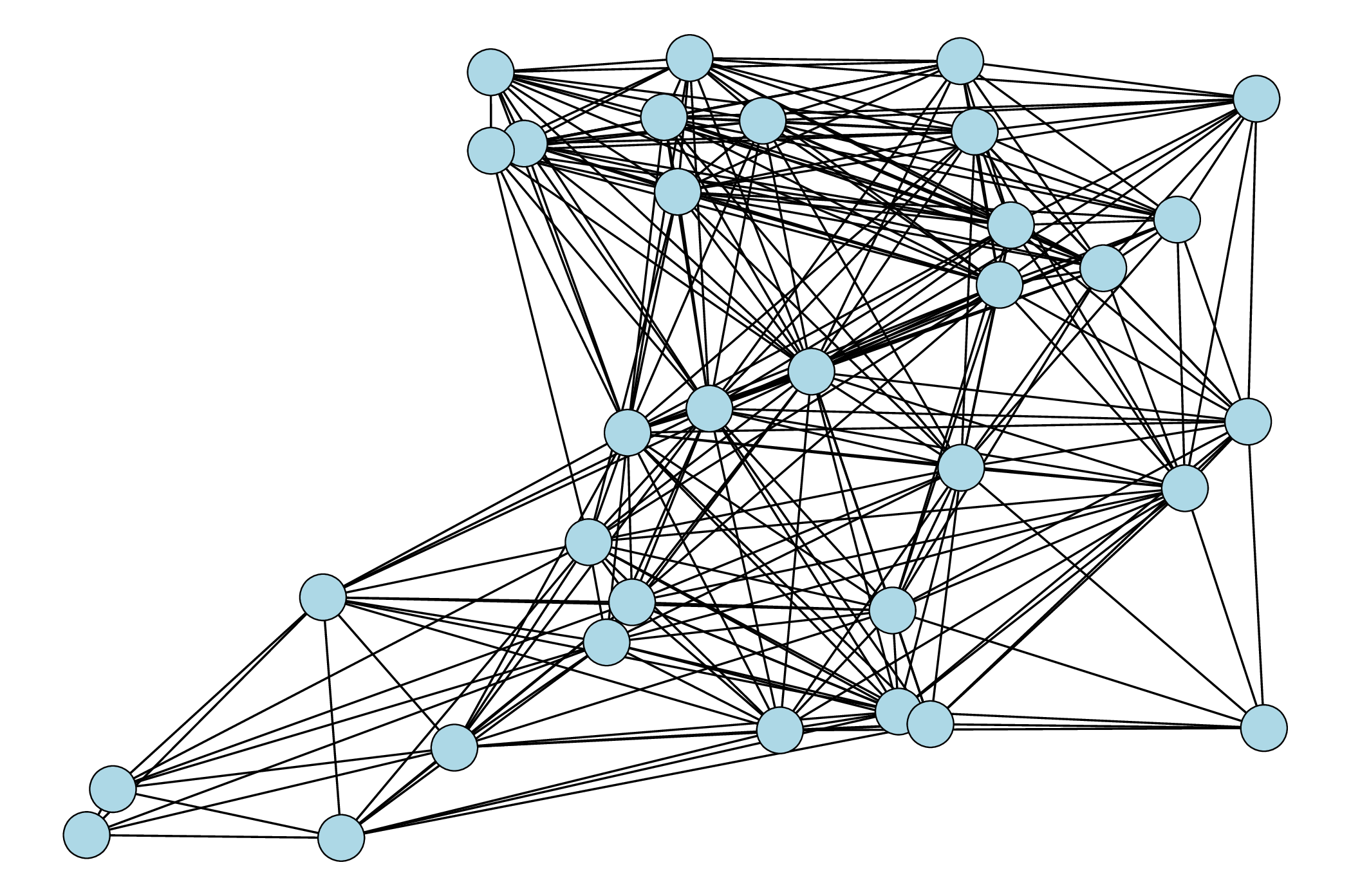}
        \vspace{-2em}
        \caption{RG}
        \label{fig: base random}
    \end{subfigure}
    \hfill
    \begin{subfigure}[b]{0.23\textwidth}
        \centering
        \includegraphics[height = 1.2in, width=1\textwidth]{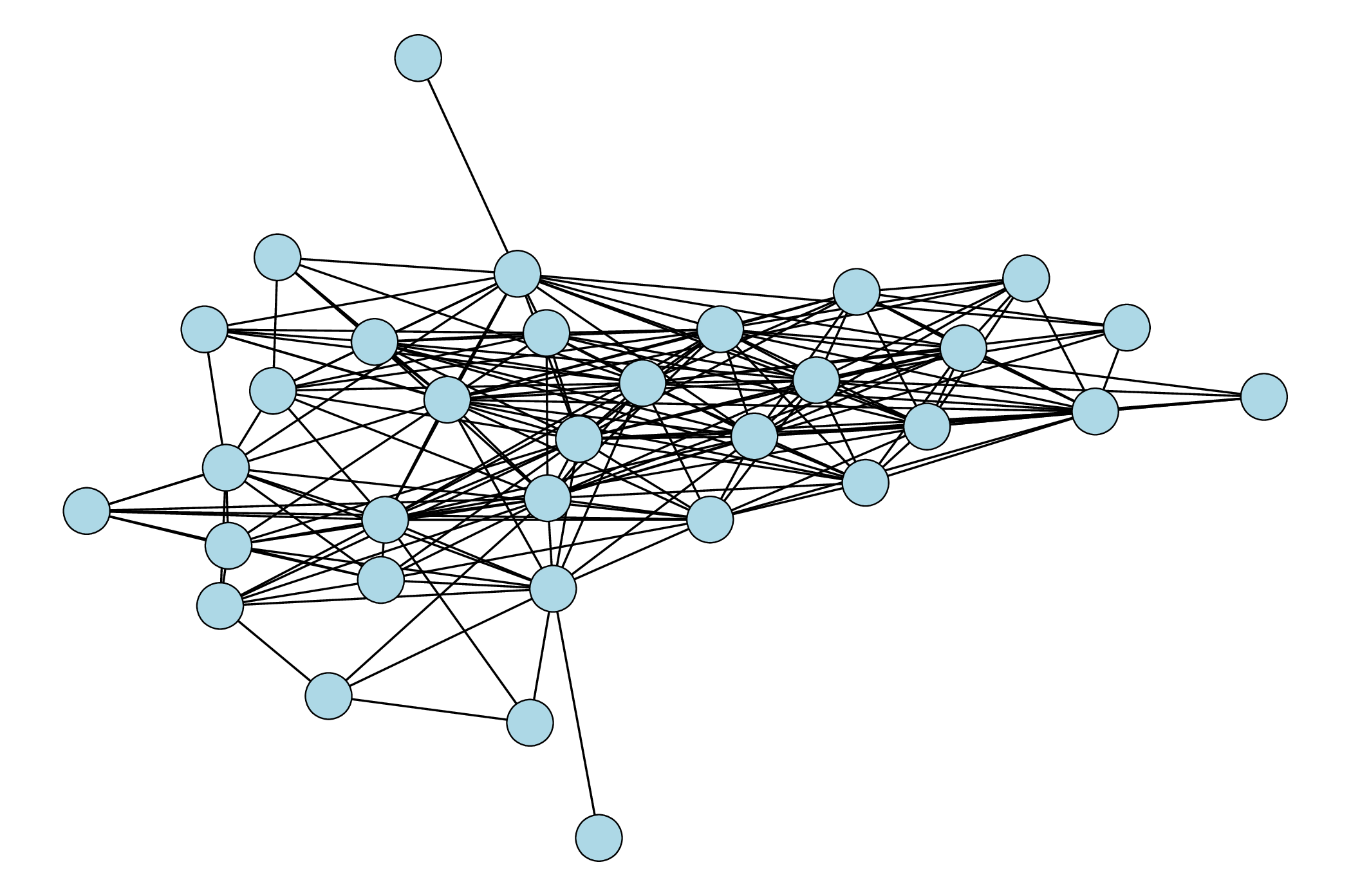}
         \vspace{-2em}
        \caption{Roofnet}
        \label{fig: base roofnet}
    \end{subfigure}
    \vspace{-1.75em}
    \caption{Base topology used in evaluation.}
    \label{fig:base topology}
    \vspace{-1em}
\end{figure}      

\subsubsection{Benchmarks}

\begin{figure}[t!]
    \centering    \centerline{\mbox{\includegraphics[height = 1.95in,width=.95\linewidth]{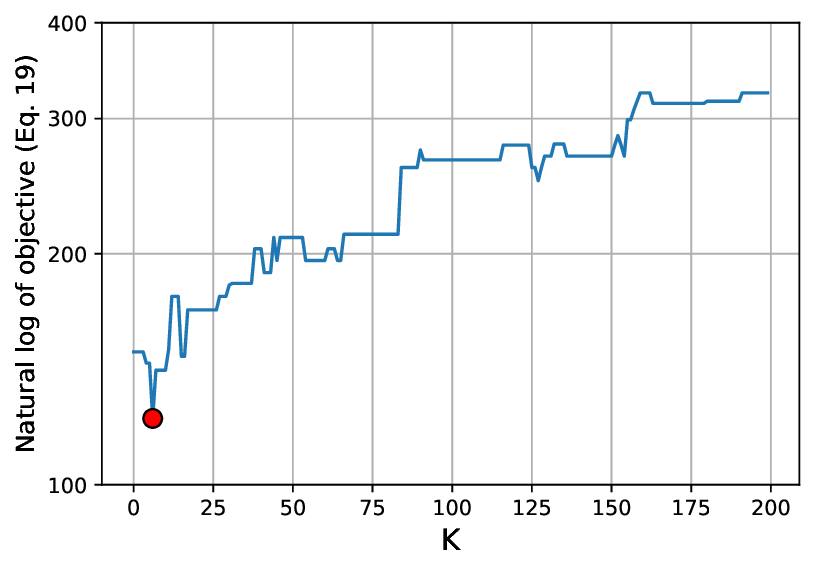}}}
    \vspace{-1.5em}
    \caption{Design objective vs. design parameter $K$ for RG ($K^*=6$). 
    }
    \label{fig:design_RG}
    \vspace{-1em}
\end{figure}

\begin{figure}[t!]
    \centering    \centerline{\mbox{\includegraphics[height = 2in,width=.95\linewidth]{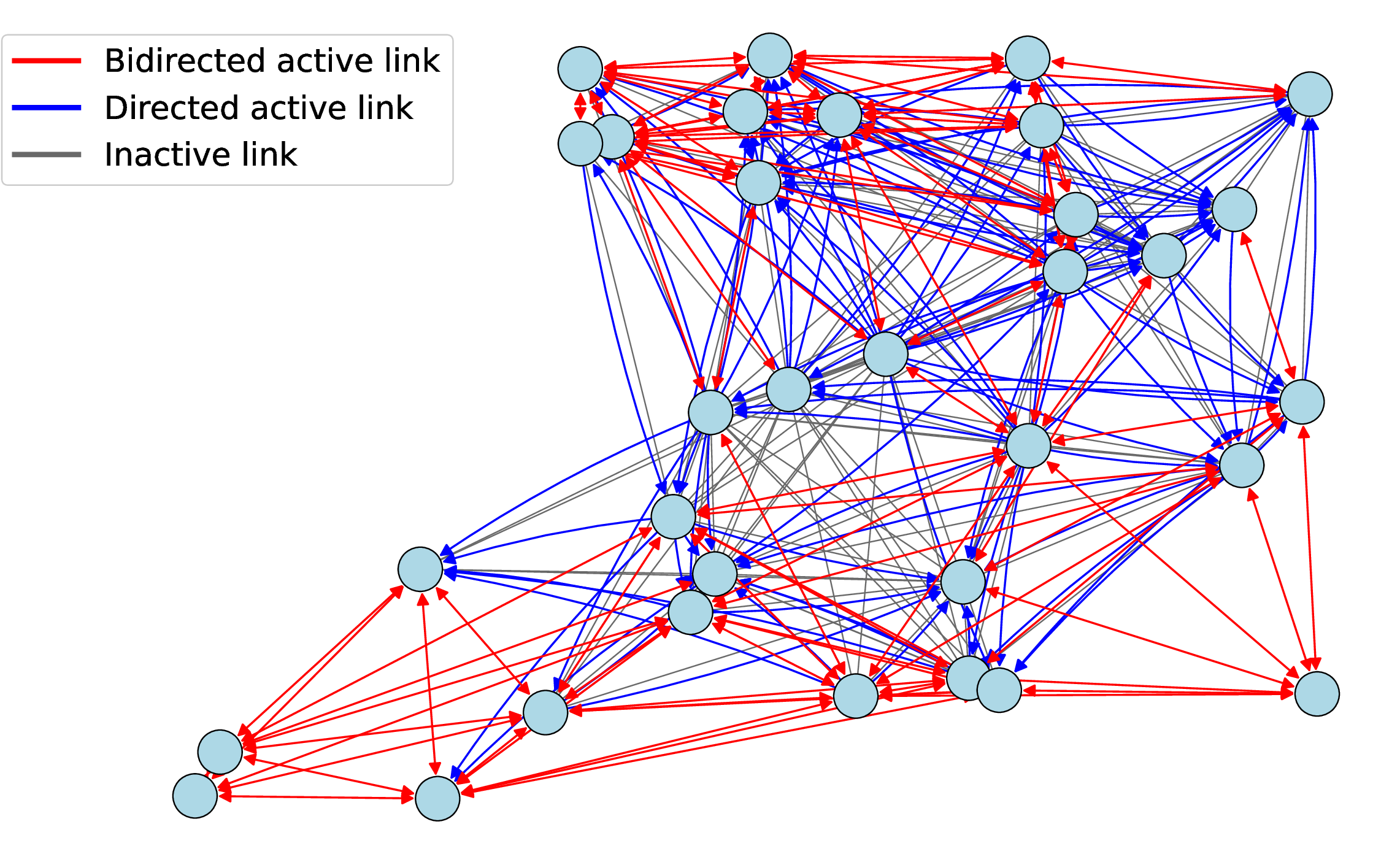}}}
    \vspace{-1.5em}
    \caption{Designed communication graph based on RG.}
    \label{fig:graph_random}
    \vspace{-1em}
\end{figure}

\begin{figure}[t!]
    \centering
    \centerline{\mbox{\includegraphics[height = 2.35in,width=1\linewidth]{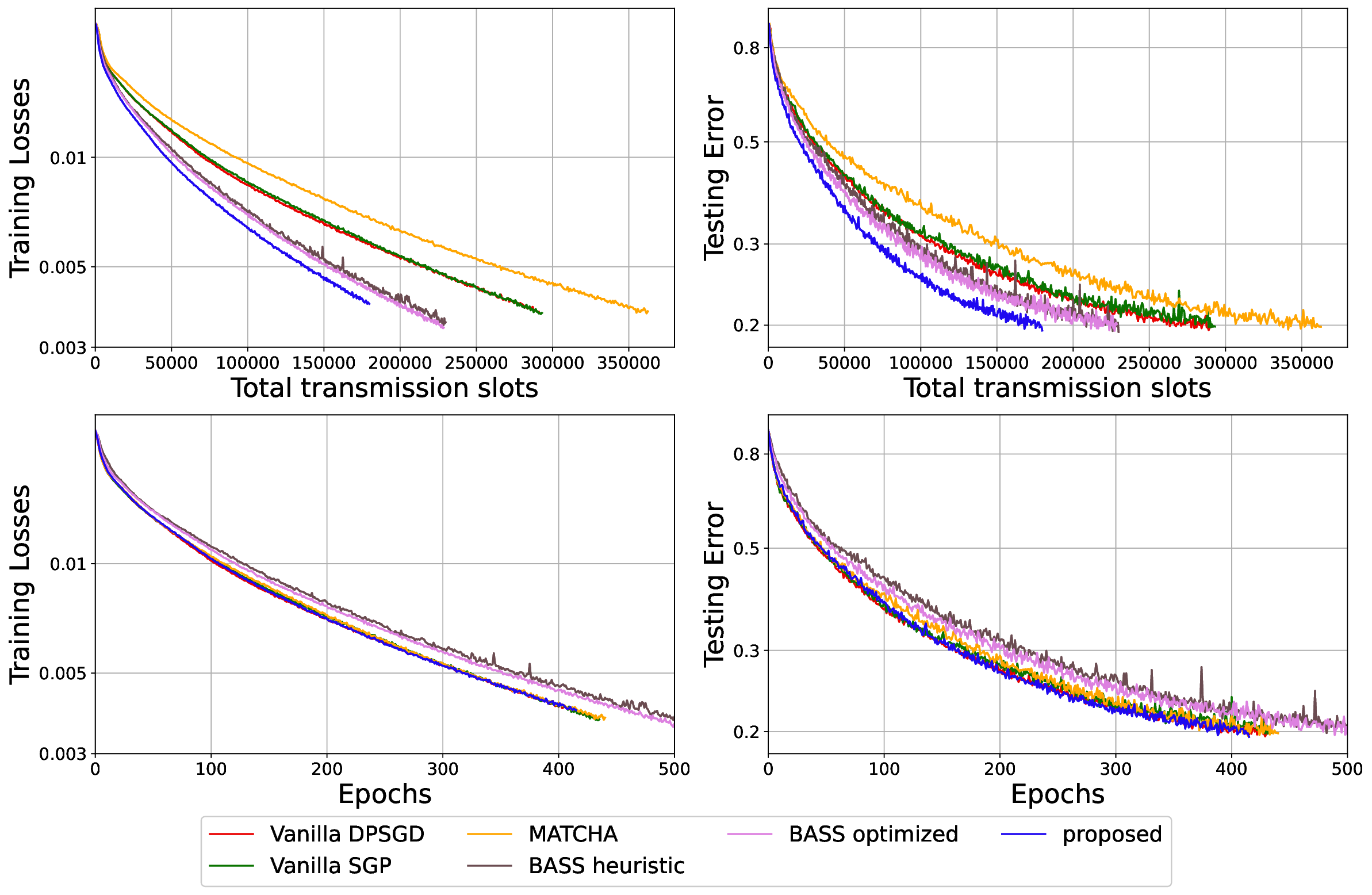}}}
    \vspace{-1em}
    \caption{Training performance on RG. 
    }
    \label{fig:results_cifar10_rg}
    \vspace{-1em}
\end{figure}

\begin{figure}[t!]
    \centering    \centerline{\mbox{\includegraphics[height = 1.95in,width=.95\linewidth]{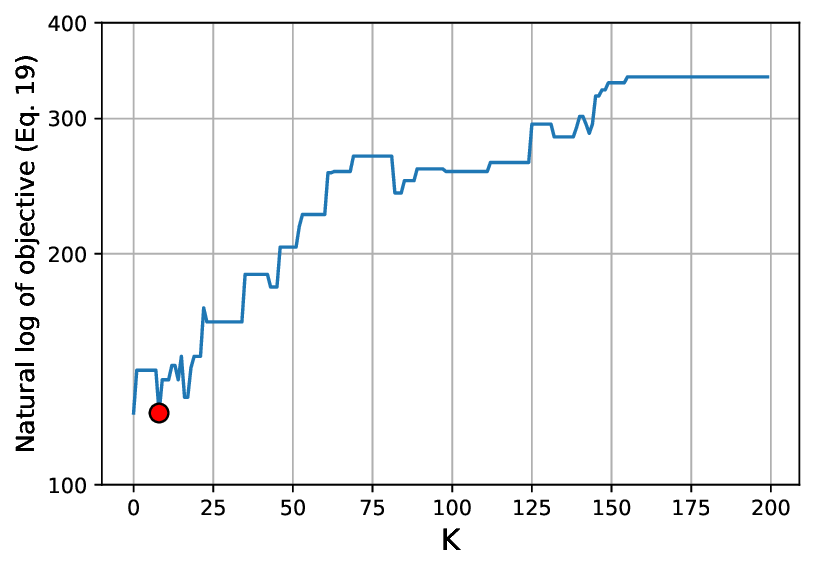}}}
    \vspace{-1.5em}
    \caption{Design objective vs. design parameter $K$ for Roofnet ($K^*=8$). 
    }
    \label{fig:design_roofnet}
    \vspace{-1em}
\end{figure}

\begin{figure}[t!]
    \centering    \centerline{\mbox{\includegraphics[height = 2in,width=.95\linewidth]{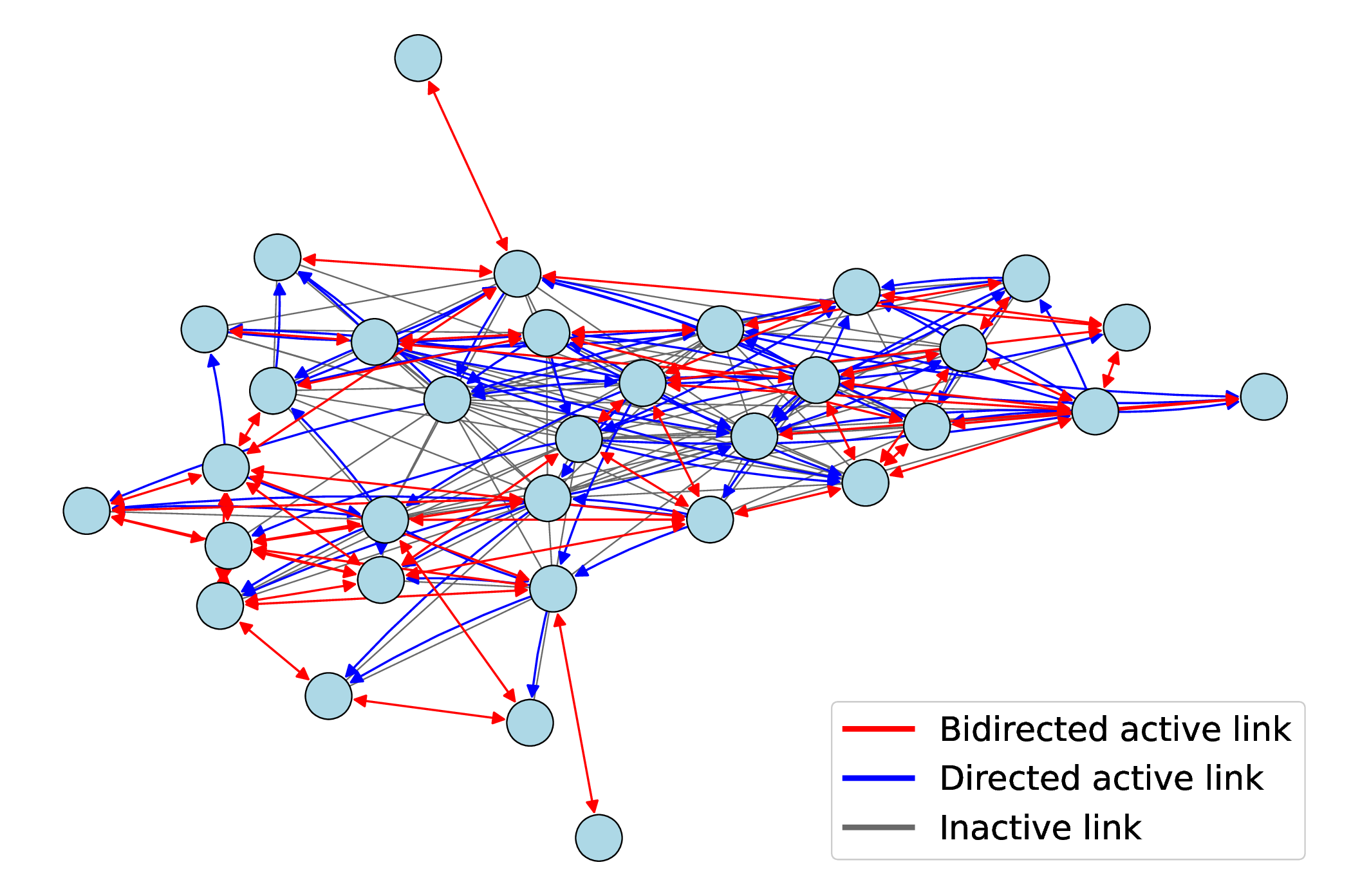}}}
    \vspace{-1.5em}
    \caption{Designed communication graph based on Roofnet. }
    \label{fig:graph_roofnet}
    \vspace{-1em}
\end{figure}

\begin{figure}[t!]
    \centering    \centerline{\mbox{\includegraphics[height = 2.35in,width=1\linewidth]{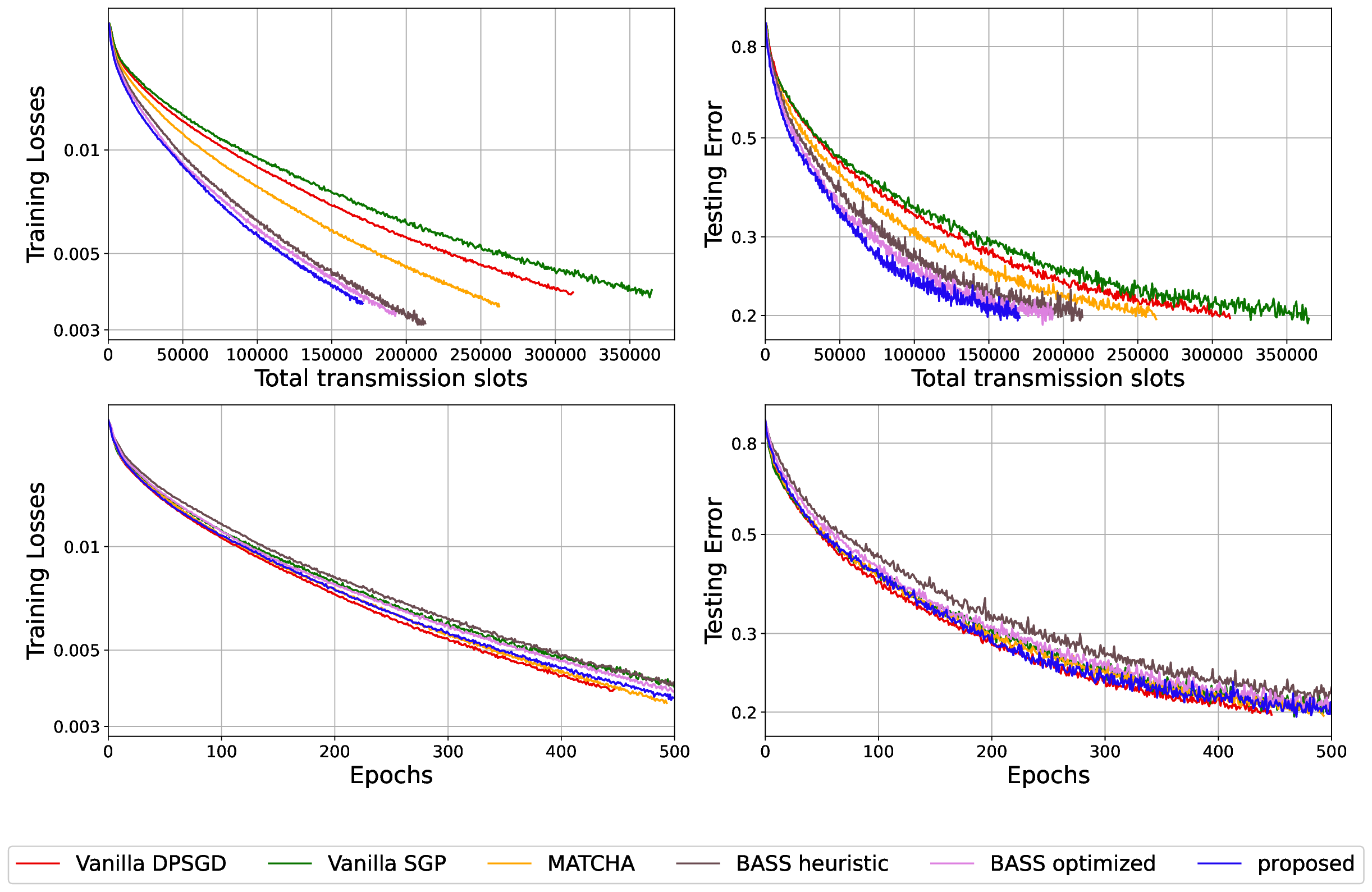}}}
    \vspace{-1em}
    \caption{Training performance on Roofnet.
        }
    \label{fig:results_cifar10_roofnet_error}
    \vspace{-1.5em}
\end{figure}

We compare the proposed solution with the following benchmarks: \begin{itemize}
    \item `Vanilla D-PSGD' \cite{Lian17NIPS}, which is a baseline that uses the base topology as the communication graph and Metropolis–Hastings weights; 
    \item `Vanilla SGP' \cite{Assran19ICML}, which uses the base topology as the communication graph but assigns weights according to Lemma~\ref{lem: uniform weight};  
    \item `MATCHA' \cite{MATCHA22}, which aims at minimizing the communication time but assumes communications on disjoint links can occur in parallel (ignoring interference); 
    \item `BASS' \cite{Herrera25OJCS} a state-of-the-art mixing matrix design for minimizing the communication time under broadcast communications by activating collision-free subsets of nodes with designed probabilities, which proposed a lightweight solution called `BASS heuristic' and a computationally heavy solution called `BASS optimized'. 
\end{itemize}
We note that all but the proposed solution and `Vanilla SGP' are based on D-PSGD and hence must use bidirected communication graphs and symmetric mixing matrices. For the solutions \cite{MATCHA22,Chiu23JSAC,Herrera25OJCS} with configurable communication cost per iteration, we align their per-iteration cost with the proposed solution\footnote{As `MATCHA' \cite{MATCHA22} are unaware of the interference constraint, we configure their per-iteration cost as a percentage of the maximum cost under their assumption and align the percentage with our solution.}. We set the number of candidate mixing matrices for `BASS optimized' to 100, comparable to the setting in \cite{Herrera25OJCS}.

\subsection{Evaluation Results}

\subsubsection{Results on Random Geometric (RG) Graph}

We first use the random geometric graph in Fig.~\ref{fig: base random} as the base topology. Fig.~\ref{fig:design_RG} evaluates the impact of the design parameter $K$ in Alg.~\ref{Alg:Graph Design} on our design objective \eqref{eq: G_a obj}, which guides the selection of an optimal parameter value $K^*$. Under this parameter, Alg.~\ref{Alg:Graph Design} designs an activated communication graph as shown in Fig.~\ref{fig:graph_random}, which contains both bidirectional and unidirectional links. 

We then evaluate the actual learning performance in terms of training loss and testing error in Fig.~\ref{fig:results_cifar10_rg}. The results show that: (i) while all the designs except for `BASS heuristic/optimized' have similar logical convergence rates measured in epochs, they differ substantially in the actual convergence rates measured in slots; (ii) our proposed solution notably outperforms the state-of-the-art solution `BASS optimized' in terms of the convergence time in slots, which in turn outperforms the other benchmarks.

\subsubsection{Results on Roofnet}

We repeat the above experiments on the Roofnet topology shown in Fig.~\ref{fig: base roofnet}. The results in Fig.~\ref{fig:design_roofnet}--\ref{fig:results_cifar10_roofnet_error} show qualitatively similar observations as before, indicating the generalizability of our previous observations. In this case, $K=0$ and $K=8$ yield similar results in both the abstract design objective \eqref{eq: G_a obj} and the actual learning performance (the curve for $K=0$ is omitted in Fig.~\ref{fig:results_cifar10_roofnet_error} for better visibility). A quantitative difference from the results on RG is that the performance gap between our solution and the state of the art (`BASS optimized') in terms of convergence rate in slots appears smaller for Roofnet. Intuitively, this is because the Roofnet topology is sparser than the RG graph used in our evaluation, which reduces the degree of freedom in the design. 
Nevertheless, in both cases, our design notably improves upon the state of the art by strategically selecting the communication links based on a theory-backed objective function.  

\subsubsection{Summary}

\begin{table}[t]
\small
    \centering
    \begin{tabular}{c|ccccc}
  Base   &   D-PSGD & MATCHA & \hspace{-.15em}BASS heu \hspace{-.15em}& BASS opt & Proposed  \\
         \hline
RG & 290,976 &   362,502 & 229,824 & 228,528   & 179,712 \\
Roofnet & 311,808 &   262,314 & 212,688 & 192,528  & 170,688 
    \end{tabular}
     \vspace{-0.25em}
    \caption{Convergence Time for 80\% Accuracy (in slots)
    }
    \label{tab:convergence_time}
    \vspace{-3em}
\end{table} 

Table~\ref{tab:convergence_time} summarizes the convergence times from these experiments, measured by the total number of transmission slots until the average testing accuracy over 5 consecutive epochs reaches 80\%. The results show that the proposed solution reduces the convergence time by 11--21\% compared to the state-of-the-art solution `BASS optimized', and 38--45\% compared to the baseline solution `Vanilla D-PSGD'. Notably, `Vanilla SGP' performs worse than `Vanilla D-PSGD' (omitted in Table~\ref{tab:convergence_time} due to space limitation), indicating that the above improvement is not from simply switching the learning algorithm from D-PSGD to SGP, but from suitably utilizing the additional communication flexibility enabled by SGP in communication graph design.

\section{Conclusion}\label{sec:Conclusion}

We considered the mixing matrix design for minimizing the convergence time in decentralized federated learning (DFL) under broadcast communications and interference constraints. Motivated by the ability of stochastic gradient push (SGP) in utilizing asymmetric parameter sharing, we developed a solution to optimize its mixing matrix by analyzing the dependency of its convergence rate on the mixing matrix and extracting a tractable design objective based on graph-theoretic parameters of the activated communication graph. Our solution not only has guaranteed performance in terms of the theoretical design objective, but also demonstrates notably faster convergence than the state of the art in evaluations based on real data, 
signaling the value in efficiently utilizing asymmetric communications for learning in wireless networks.    
 \looseness=0



\bibliographystyle{ACM-Reference-Format}
\bibliography{references}


\begin{thebibliography}{35}


\ifx \showCODEN    \undefined \def \showCODEN     #1{\unskip}     \fi
\ifx \showISBNx    \undefined \def \showISBNx     #1{\unskip}     \fi
\ifx \showISBNxiii \undefined \def \showISBNxiii  #1{\unskip}     \fi
\ifx \showISSN     \undefined \def \showISSN      #1{\unskip}     \fi
\ifx \showLCCN     \undefined \def \showLCCN      #1{\unskip}     \fi
\ifx \shownote     \undefined \def \shownote      #1{#1}          \fi
\ifx \showarticletitle \undefined \def \showarticletitle #1{#1}   \fi
\ifx \showURL      \undefined \def \showURL       {\relax}        \fi
\providecommand\bibfield[2]{#2}
\providecommand\bibinfo[2]{#2}
\providecommand\natexlab[1]{#1}
\providecommand\showeprint[2][]{arXiv:#2}

\bibitem[Aguayo et~al\mbox{.}(2004)]%
        {Roofnet}
\bibfield{author}{\bibinfo{person}{Daniel Aguayo}, \bibinfo{person}{John Bicket}, \bibinfo{person}{Sanjit Biswas}, \bibinfo{person}{Glenn Judd}, {and} \bibinfo{person}{Robert Morris}.} \bibinfo{year}{2004}\natexlab{}.
\newblock \showarticletitle{Link-level Measurements from an 802.11b Mesh Network}. In \bibinfo{booktitle}{\emph{SIGCOMM}}.
\newblock


\bibitem[Assran et~al\mbox{.}(2019)]%
        {Assran19ICML}
\bibfield{author}{\bibinfo{person}{Mahmoud Assran}, \bibinfo{person}{Nicolas Loizou}, \bibinfo{person}{Nicolas Ballas}, {and} \bibinfo{person}{Mike Rabbat}.} \bibinfo{year}{2019}\natexlab{}.
\newblock \showarticletitle{Stochastic Gradient Push for Distributed Deep Learning}. In \bibinfo{booktitle}{\emph{Proceedings of the 36th International Conference on Machine Learning}}, Vol.~\bibinfo{volume}{97}. \bibinfo{pages}{344--353}.
\newblock


\bibitem[Chen et~al\mbox{.}(2022)]%
        {chen2022federated}
\bibfield{author}{\bibinfo{person}{Xianhao Chen}, \bibinfo{person}{Guangyu Zhu}, \bibinfo{person}{Yiqin Deng}, {and} \bibinfo{person}{Yuguang Fang}.} \bibinfo{year}{2022}\natexlab{}.
\newblock \showarticletitle{Federated learning over multihop wireless networks with in-network aggregation}.
\newblock \bibinfo{journal}{\emph{IEEE Transactions on Wireless Communications}} \bibinfo{volume}{21}, \bibinfo{number}{6} (\bibinfo{year}{2022}), \bibinfo{pages}{4622--4634}.
\newblock


\bibitem[Chen et~al\mbox{.}(2023)]%
        {Chen23SPAWC}
\bibfield{author}{\bibinfo{person}{Zheng Chen}, \bibinfo{person}{Martin Dahl}, {and} \bibinfo{person}{Erik~G. Larsson}.} \bibinfo{year}{2023}\natexlab{}.
\newblock \showarticletitle{Decentralized Learning over Wireless Networks: The Effect of Broadcast with Random Access}. In \bibinfo{booktitle}{\emph{2023 IEEE 24th International Workshop on Signal Processing Advances in Wireless Communications (SPAWC)}}. \bibinfo{pages}{316--320}.
\newblock
\href{https://doi.org/10.1109/SPAWC53906.2023.10304514}{doi:\nolinkurl{10.1109/SPAWC53906.2023.10304514}}


\bibitem[Chiu et~al\mbox{.}(2023)]%
        {Chiu23JSAC}
\bibfield{author}{\bibinfo{person}{Cho-Chun Chiu}, \bibinfo{person}{Xusheng Zhang}, \bibinfo{person}{Ting He}, \bibinfo{person}{Shiqiang Wang}, {and} \bibinfo{person}{Ananthram Swami}.} \bibinfo{year}{2023}\natexlab{}.
\newblock \showarticletitle{Laplacian Matrix Sampling for Communication- Efficient Decentralized Learning}.
\newblock \bibinfo{journal}{\emph{IEEE Journal on Selected Areas in Communications}} \bibinfo{volume}{41}, \bibinfo{number}{4} (\bibinfo{year}{2023}), \bibinfo{pages}{887--901}.
\newblock
\href{https://doi.org/10.1109/JSAC.2023.3242735}{doi:\nolinkurl{10.1109/JSAC.2023.3242735}}


\bibitem[Diestel(2017)]%
        {Diestel2017GraphTheory}
\bibfield{author}{\bibinfo{person}{Reinhard Diestel}.} \bibinfo{year}{2017}\natexlab{}.
\newblock \bibinfo{booktitle}{\emph{Graph Theory} (\bibinfo{edition}{5} ed.)}.
\newblock \bibinfo{publisher}{Springer}.
\newblock


\bibitem[Fürer and Raghavachari(1992)]%
        {Furer92SODA}
\bibfield{author}{\bibinfo{person}{M. Fürer} {and} \bibinfo{person}{B. Raghavachari}.} \bibinfo{year}{1992}\natexlab{}.
\newblock \showarticletitle{Approximating the Minimum Degree Spanning Tree to Within One from the Optimal Degree}. In \bibinfo{booktitle}{\emph{Proceedings of the 3rd Annual ACM–SIAM Symposium on Discrete Algorithms (SODA)}}. \bibinfo{pages}{317--324}.
\newblock


\bibitem[Garey and Johnson(1979)]%
        {GareyJohnson1979}
\bibfield{author}{\bibinfo{person}{Michael~R. Garey} {and} \bibinfo{person}{David~S. Johnson}.} \bibinfo{year}{1979}\natexlab{}.
\newblock \bibinfo{booktitle}{\emph{Computers and Intractability: A Guide to the Theory of NP-Completeness}}.
\newblock \bibinfo{publisher}{W. H. Freeman}.
\newblock


\bibitem[Herrera et~al\mbox{.}(2025)]%
        {Herrera25OJCS}
\bibfield{author}{\bibinfo{person}{Daniel~PÉRez Herrera}, \bibinfo{person}{Zheng Chen}, {and} \bibinfo{person}{Erik~G. Larsson}.} \bibinfo{year}{2025}\natexlab{}.
\newblock \showarticletitle{Faster Convergence With Less Communication: Broadcast-Based Subgraph Sampling for Decentralized Learning Over Wireless Networks}.
\newblock \bibinfo{journal}{\emph{IEEE Open Journal of the Communications Society}} (\bibinfo{year}{2025}), \bibinfo{pages}{1--1}.
\newblock
\href{https://doi.org/10.1109/OJCOMS.2025.3540133}{doi:\nolinkurl{10.1109/OJCOMS.2025.3540133}}


\bibitem[Hua et~al\mbox{.}(2022)]%
        {hua2022efficient}
\bibfield{author}{\bibinfo{person}{Yifan Hua}, \bibinfo{person}{Kevin Miller}, \bibinfo{person}{Andrea~L Bertozzi}, \bibinfo{person}{Chen Qian}, {and} \bibinfo{person}{Bao Wang}.} \bibinfo{year}{2022}\natexlab{}.
\newblock \showarticletitle{Efficient and reliable overlay networks for decentralized federated learning}.
\newblock \bibinfo{journal}{\emph{SIAM J. Appl. Math.}} \bibinfo{volume}{82}, \bibinfo{number}{4} (\bibinfo{year}{2022}), \bibinfo{pages}{1558--1586}.
\newblock


\bibitem[Jiang et~al\mbox{.}(2023)]%
        {jiang2023joint}
\bibfield{author}{\bibinfo{person}{Zhida Jiang}, \bibinfo{person}{Yang Xu}, \bibinfo{person}{Hongli Xu}, \bibinfo{person}{Lun Wang}, \bibinfo{person}{Chunming Qiao}, {and} \bibinfo{person}{Liusheng Huang}.} \bibinfo{year}{2023}\natexlab{}.
\newblock \showarticletitle{Joint Model Pruning and Topology Construction for Accelerating Decentralized Machine Learning}.
\newblock \bibinfo{journal}{\emph{IEEE Transactions on Parallel and Distributed Systems}} (\bibinfo{year}{2023}).
\newblock


\bibitem[Koloskova et~al\mbox{.}(2020)]%
        {Compression1}
\bibfield{author}{\bibinfo{person}{Anastasia Koloskova}, \bibinfo{person}{Tao Lin}, \bibinfo{person}{Sebastian~U Stich}, {and} \bibinfo{person}{Martin Jagg}.} \bibinfo{year}{2020}\natexlab{}.
\newblock \showarticletitle{Decentralized Deep Learning with Arbitrary Communication Compression}. In \bibinfo{booktitle}{\emph{The International Conference on Learning Representations (ICLR)}}.
\newblock


\bibitem[Le~Bars et~al\mbox{.}(2023)]%
        {le2023refined}
\bibfield{author}{\bibinfo{person}{Batiste Le~Bars}, \bibinfo{person}{Aur{\'e}lien Bellet}, \bibinfo{person}{Marc Tommasi}, \bibinfo{person}{Erick Lavoie}, {and} \bibinfo{person}{Anne-Marie Kermarrec}.} \bibinfo{year}{2023}\natexlab{}.
\newblock \showarticletitle{Refined convergence and topology learning for decentralized {SGD} with heterogeneous data}. In \bibinfo{booktitle}{\emph{International Conference on Artificial Intelligence and Statistics}}. PMLR, \bibinfo{pages}{1672--1702}.
\newblock


\bibitem[Lian et~al\mbox{.}(2017)]%
        {Lian17NIPS}
\bibfield{author}{\bibinfo{person}{Xiangru Lian}, \bibinfo{person}{Ce Zhang}, \bibinfo{person}{Huan Zhang}, \bibinfo{person}{Cho-Jui Hsieh}, \bibinfo{person}{Wei Zhang}, {and} \bibinfo{person}{Ji Liu}.} \bibinfo{year}{2017}\natexlab{}.
\newblock \showarticletitle{Can Decentralized Algorithms Outperform Centralized Algorithms? A Case Study for Decentralized Parallel Stochastic Gradient Descent}. In \bibinfo{booktitle}{\emph{Proceedings of the 31st International Conference on Neural Information Processing Systems}}. \bibinfo{pages}{5336–5346}.
\newblock


\bibitem[Liang et~al\mbox{.}(2025)]%
        {Liang25JO}
\bibfield{author}{\bibinfo{person}{Liyuan Liang}, \bibinfo{person}{Xinmeng Huang}, \bibinfo{person}{Ran Xin}, {and} \bibinfo{person}{Kun Yuan}.} \bibinfo{year}{2025}\natexlab{}.
\newblock \showarticletitle{Understanding the Influence of Digraphs on Decentralized Optimization: Effective Metrics, Lower Bound, and Optimal Algorithm}.
\newblock \bibinfo{journal}{\emph{SIAM Journal on Optimization}} \bibinfo{volume}{35}, \bibinfo{number}{3} (\bibinfo{year}{2025}), \bibinfo{pages}{1570--1600}.
\newblock


\bibitem[Lu and Sa(2020)]%
        {Compression2}
\bibfield{author}{\bibinfo{person}{Yucheng Lu} {and} \bibinfo{person}{Christopher~De Sa}.} \bibinfo{year}{2020}\natexlab{}.
\newblock \showarticletitle{{Moniqua}: Modulo Quantized Communication in Decentralized {SGD}}. In \bibinfo{booktitle}{\emph{International Conference on Machine Learning (ICML)}}.
\newblock


\bibitem[Lu and Sa(2021)]%
        {ICMLhonor}
\bibfield{author}{\bibinfo{person}{Yucheng Lu} {and} \bibinfo{person}{Christopher~De Sa}.} \bibinfo{year}{2021}\natexlab{}.
\newblock \showarticletitle{Optimal Complexity in Decentralized Training}. In \bibinfo{booktitle}{\emph{International Conference on Machine Learning (ICML)}}.
\newblock


\bibitem[McMahan et~al\mbox{.}(2017)]%
        {McMahan17AISTATS}
\bibfield{author}{\bibinfo{person}{H. McMahan}, \bibinfo{person}{Eider Moore}, \bibinfo{person}{D. Ramage}, \bibinfo{person}{S. Hampson}, {and} \bibinfo{person}{Blaise~Ag{\"u}era y Arcas}.} \bibinfo{year}{2017}\natexlab{}.
\newblock \showarticletitle{Communication-Efficient Learning of Deep Networks from Decentralized Data}. In \bibinfo{booktitle}{\emph{AISTATS}}.
\newblock


\bibitem[Neglia et~al\mbox{.}(2020)]%
        {neglia2020decentralized}
\bibfield{author}{\bibinfo{person}{Giovanni Neglia}, \bibinfo{person}{Chuan Xu}, \bibinfo{person}{Don Towsley}, {and} \bibinfo{person}{Gianmarco Calbi}.} \bibinfo{year}{2020}\natexlab{}.
\newblock \showarticletitle{Decentralized gradient methods: does topology matter?}. In \bibinfo{booktitle}{\emph{International Conference on Artificial Intelligence and Statistics}}. PMLR, \bibinfo{pages}{2348--2358}.
\newblock


\bibitem[Singh et~al\mbox{.}(2020)]%
        {Singh20CDC}
\bibfield{author}{\bibinfo{person}{Navjot Singh}, \bibinfo{person}{Deepesh Data}, \bibinfo{person}{Jemin George}, {and} \bibinfo{person}{Suhas Diggavi}.} \bibinfo{year}{2020}\natexlab{}.
\newblock \showarticletitle{{SPARQ-SGD}: Event-Triggered and Compressed Communication in Decentralized Optimization}. In \bibinfo{booktitle}{\emph{IEEE CDC}}.
\newblock


\bibitem[Singh et~al\mbox{.}(2021)]%
        {Singh21JSAIT}
\bibfield{author}{\bibinfo{person}{Navjot Singh}, \bibinfo{person}{Deepesh Data}, \bibinfo{person}{Jemin George}, {and} \bibinfo{person}{Suhas Diggavi}.} \bibinfo{year}{2021}\natexlab{}.
\newblock \showarticletitle{{SQuARM-SGD}: Communication-Efficient Momentum {SGD} for Decentralized Optimization}.
\newblock \bibinfo{journal}{\emph{IEEE Journal on Selected Areas in Information Theory}} \bibinfo{volume}{2}, \bibinfo{number}{3} (\bibinfo{year}{2021}), \bibinfo{pages}{954--969}.
\newblock


\bibitem[Spiridonoff et~al\mbox{.}(2020)]%
        {Spiridonoff20JMLR}
\bibfield{author}{\bibinfo{person}{Artin Spiridonoff}, \bibinfo{person}{Alex Olshevsky}, {and} \bibinfo{person}{Ioannis~Ch. Paschalidis}.} \bibinfo{year}{2020}\natexlab{}.
\newblock \showarticletitle{Robust Asynchronous Stochastic Gradient-Push: Asymptotically Optimal and Network-Independent Performance for Strongly Convex Functions}.
\newblock \bibinfo{journal}{\emph{Journal of Machine Learning Research}} \bibinfo{volume}{21}, \bibinfo{number}{58} (\bibinfo{year}{2020}), \bibinfo{pages}{1--47}.
\newblock


\bibitem[Tang et~al\mbox{.}(2018a)]%
        {Compression3}
\bibfield{author}{\bibinfo{person}{Hanlin Tang}, \bibinfo{person}{Shaoduo Gan}, \bibinfo{person}{Ce Zhang}, \bibinfo{person}{Tong Zhang}, {and} \bibinfo{person}{Ji Liu}.} \bibinfo{year}{2018}\natexlab{a}.
\newblock \showarticletitle{Communication compression for decentralized training}. In \bibinfo{booktitle}{\emph{Advances in Neural Information Processing Systems (NeurIPS)}}.
\newblock


\bibitem[Tang et~al\mbox{.}(2018b)]%
        {D2ICML18}
\bibfield{author}{\bibinfo{person}{Hanlin Tang}, \bibinfo{person}{Xiangru Lian}, \bibinfo{person}{Ming Yan}, \bibinfo{person}{Ce Zhang}, {and} \bibinfo{person}{Ji Liu}.} \bibinfo{year}{2018}\natexlab{b}.
\newblock \showarticletitle{$D^2$: Decentralized Training over Decentralized Data}. In \bibinfo{booktitle}{\emph{Proceedings of the 35th International Conference on Machine Learning, {ICML}}}.
\newblock


\bibitem[Tran et~al\mbox{.}(2019)]%
        {Ngu19INFOCOM}
\bibfield{author}{\bibinfo{person}{Nguyen~H. Tran}, \bibinfo{person}{Wei Bao}, \bibinfo{person}{Albert Zomaya}, \bibinfo{person}{Minh~N.H. Nguyen}, {and} \bibinfo{person}{Choong~Seon Hong}.} \bibinfo{year}{2019}\natexlab{}.
\newblock \showarticletitle{Federated Learning over Wireless Networks: Optimization Model Design and Analysis}. In \bibinfo{booktitle}{\emph{IEEE INFOCOM}}.
\newblock


\bibitem[Vogels et~al\mbox{.}(2022)]%
        {vogels2022beyond}
\bibfield{author}{\bibinfo{person}{Thijs Vogels}, \bibinfo{person}{Hadrien Hendrikx}, {and} \bibinfo{person}{Martin Jaggi}.} \bibinfo{year}{2022}\natexlab{}.
\newblock \showarticletitle{Beyond spectral gap: The role of the topology in decentralized learning}.
\newblock \bibinfo{journal}{\emph{Advances in Neural Information Processing Systems}}  \bibinfo{volume}{35} (\bibinfo{year}{2022}), \bibinfo{pages}{15039--15050}.
\newblock


\bibitem[Wang and Joshi(2019)]%
        {sysml19}
\bibfield{author}{\bibinfo{person}{Jianyu Wang} {and} \bibinfo{person}{Gauri Joshi}.} \bibinfo{year}{2019}\natexlab{}.
\newblock \showarticletitle{Adaptive Communication Strategies to Achieve the Best Error-Runtime Trade-off in Local-Update {SGD}}. In \bibinfo{booktitle}{\emph{Systems for ML}}.
\newblock


\bibitem[Wang et~al\mbox{.}(2022)]%
        {MATCHA22}
\bibfield{author}{\bibinfo{person}{Jianyu Wang}, \bibinfo{person}{Anit~Kumar Sahu}, \bibinfo{person}{Gauri Joshi}, {and} \bibinfo{person}{Soummya Kar}.} \bibinfo{year}{2022}\natexlab{}.
\newblock \showarticletitle{{MATCHA}: A Matching-Based Link Scheduling Strategy to Speed up Distributed Optimization}.
\newblock \bibinfo{journal}{\emph{IEEE Transactions on Signal Processing}}  \bibinfo{volume}{70} (\bibinfo{year}{2022}), \bibinfo{pages}{5208--5221}.
\newblock


\bibitem[Wang et~al\mbox{.}(2019)]%
        {Wang19JSAC}
\bibfield{author}{\bibinfo{person}{Shiqiang Wang}, \bibinfo{person}{Tiffany Tuor}, \bibinfo{person}{Theodoros Salonidis}, \bibinfo{person}{Kin~K. Leung}, \bibinfo{person}{Christian Makaya}, \bibinfo{person}{Ting He}, {and} \bibinfo{person}{Kevin Chan}.} \bibinfo{year}{2019}\natexlab{}.
\newblock \showarticletitle{Adaptive Federated Learning in Resource Constrained Edge Computing Systems}.
\newblock \bibinfo{journal}{\emph{IEEE Journal on Selected Areas in Communications}} \bibinfo{volume}{37}, \bibinfo{number}{6} (\bibinfo{year}{2019}), \bibinfo{pages}{1205--1221}.
\newblock
\href{https://doi.org/10.1109/JSAC.2019.2904348}{doi:\nolinkurl{10.1109/JSAC.2019.2904348}}


\bibitem[Xiao et~al\mbox{.}(2017)]%
        {FMNIST}
\bibfield{author}{\bibinfo{person}{Han Xiao}, \bibinfo{person}{Kashif Rasul}, {and} \bibinfo{person}{Roland Vollgraf}.} \bibinfo{year}{2017}\natexlab{}.
\newblock \showarticletitle{Fashion-MNIST: a Novel Image Dataset for Benchmarking Machine Learning Algorithms}.
\newblock \bibinfo{journal}{\emph{arXiv preprint arXiv:1708.07747}} (\bibinfo{year}{2017}).
\newblock


\bibitem[Xin et~al\mbox{.}(2021)]%
        {Xin21}
\bibfield{author}{\bibinfo{person}{Ran Xin}, \bibinfo{person}{Usman Khan}, {and} \bibinfo{person}{Soummya Kar}.} \bibinfo{year}{2021}\natexlab{}.
\newblock \showarticletitle{A Hybrid Variance-Reduced Method for Decentralized Stochastic Non-Convex Optimization}. In \bibinfo{booktitle}{\emph{Proceedings of the 38th International Conference on Machine Learning}} \emph{(\bibinfo{series}{Proceedings of Machine Learning Research}, Vol.~\bibinfo{volume}{139})}, \bibfield{editor}{\bibinfo{person}{Marina Meila} {and} \bibinfo{person}{Tong Zhang}} (Eds.). \bibinfo{publisher}{PMLR}, \bibinfo{pages}{11459--11469}.
\newblock
\urldef\tempurl%
\url{https://proceedings.mlr.press/v139/xin21a.html}
\showURL{%
\tempurl}


\bibitem[Xing et~al\mbox{.}(2021)]%
        {xing2021federated}
\bibfield{author}{\bibinfo{person}{Hong Xing}, \bibinfo{person}{Osvaldo Simeone}, {and} \bibinfo{person}{Suzhi Bi}.} \bibinfo{year}{2021}\natexlab{}.
\newblock \showarticletitle{Federated learning over wireless device-to-device networks: Algorithms and convergence analysis}.
\newblock \bibinfo{journal}{\emph{IEEE Journal on Selected Areas in Communications}} \bibinfo{volume}{39}, \bibinfo{number}{12} (\bibinfo{year}{2021}), \bibinfo{pages}{3723--3741}.
\newblock


\bibitem[You and Pu(2025)]%
        {You25arXiv}
\bibfield{author}{\bibinfo{person}{Runze You} {and} \bibinfo{person}{Shi Pu}.} \bibinfo{year}{2025}\natexlab{}.
\newblock \bibinfo{title}{Stochastic Push-Pull for Decentralized Nonconvex Optimization}.
\newblock
\showeprint[arxiv]{2506.07021}~[math.OC]
\urldef\tempurl%
\url{https://arxiv.org/abs/2506.07021}
\showURL{%
\tempurl}


\bibitem[Zhang et~al\mbox{.}(2026)]%
        {Zhang26INFOCOM}
\bibfield{author}{\bibinfo{person}{Xusheng Zhang}, \bibinfo{person}{Tuan Nguyen}, {and} \bibinfo{person}{Ting He}.} \bibinfo{year}{2026}\natexlab{}.
\newblock \showarticletitle{Mixing Matrix Design for Energy-efficient Decentralized Federated Learning}. In \bibinfo{booktitle}{\emph{IEEE INFOCOM}}.
\newblock


\bibitem[Zhou et~al\mbox{.}(2025)]%
        {Zhou25TCNS}
\bibfield{author}{\bibinfo{person}{Yiming Zhou}, \bibinfo{person}{Yifei Cheng}, \bibinfo{person}{Linli Xu}, {and} \bibinfo{person}{Enhong Chen}.} \bibinfo{year}{2025}\natexlab{}.
\newblock \showarticletitle{Adaptive Weighting Push-SUM for Decentralized Optimization With Statistical Diversity}.
\newblock \bibinfo{journal}{\emph{IEEE Transactions on Control of Network Systems}} \bibinfo{volume}{12}, \bibinfo{number}{3} (\bibinfo{year}{2025}), \bibinfo{pages}{2337--2349}.
\newblock
\href{https://doi.org/10.1109/TCNS.2025.3566329}{doi:\nolinkurl{10.1109/TCNS.2025.3566329}}


\end{thebibliography}

\if\thisismainpaper1

\else
\appendix
\section{Appendix}\label{appendix:A}

\subsection{Proof of Theorem~\ref{thm:SGP convergence}}\label{appendix:Convergence Proof}
Let $\xbart := \frac{1}{n}\sum_{j=1}^{N}\xbf_j^{(t)}$. 

\subsubsection{Main Theorem Proof}
\begin{proof}    
Assume
\[
\eta \le \min\left\{
\frac{1-q}{3\sqrt{2}\,CL\sqrt n},\,1
\right\}. 
\]
Then
\[
P
:=
1-\frac{9\eta^2 C^2 L^2 n}{(1-q)^2}
\geq
\frac12.
\]
Moreover, since \(n\ge 2\),
\[
\eta\leq \frac{1-q}{3\sqrt{2}\,CL\sqrt n} \leq
\frac{1-q}{6LC},
\]
which implies 
\[
\frac{9\eta^2 L^2 C^2}{P(1-q)^2}
\le
\frac{18\eta^2 L^2 C^2}{(1-q)^2}
\le
\frac12.
\]
Hence,
\[
1-\frac{9\eta^2 L^2 C^2}{P(1-q)^2}\ge \frac12.
\]

By Lemma~\ref{lem: final avg bound},
\begin{align*}
\left(1-\frac{9\eta^2 L^2 C^2}{P(1-q)^2}\right)
\sum_{t=0}^{T-1}\EE{\norm{\nabfxbar}^2}
\le
\frac{2(f(\xbar^{(0)})-f^*)}{\eta}
+
\frac{L\eta\sigma^2T}{n}\\
\quad
+
\frac{3L^2 C^2(\maxx)^2}{P(1-q)^2}
+
\frac{3\eta^2 L^2 C^2 n\sigma^2T}{P(1-q)^2}
+
\frac{9\eta^2 L^2 C^2 n\zeta^2 T}{P(1-q)^2}.
\end{align*}
Substituting the bound
\(P\geq \frac{1}{2}\) and \(
1-\frac{9\eta^2 L^2 C^2}{P(1-q)^2}\ge \frac12
\)
into the above inequality, we obtain
\begin{align*}
&\sum_{t=0}^{T-1}\EE{\norm{\nabfxbar}^2}
\leq
\frac{4(f(\xbar^{(0)})-f^*)}{\eta}
+
\frac{2L\eta\sigma^2T}{n}\\
& +
\frac{12L^2 C^2(\maxx)^2}{(1-q)^2}
+
\frac{12\eta^2 L^2 C^2 n\sigma^2T}{(1-q)^2}
+
\frac{36\eta^2 L^2 C^2 n\zeta^2 T}{(1-q)^2}.
\end{align*}
Dividing both sides by \(T\), we obtain
\begin{align*}
&\frac{1}{T}\sum_{t=0}^{T-1}\EE{\norm{\nabfxbar}^2}
\leq
\frac{4(f(\xbar^{(0)})-f^*)}{\eta T}
+
\frac{2L\eta\sigma^2}{n}\\
&\quad
+
\frac{12L^2 C^2(\maxx)^2}{(1-q)^2T}
+
\frac{12\eta^2 L^2 C^2 n\sigma^2}{(1-q)^2}
+
\frac{36\eta^2 L^2 C^2 n\zeta^2}{(1-q)^2}.
\end{align*}
Now let \(\eta = \sqrt{\frac{n}{T}}\) and assume that \(\sqrt{\frac{n}{T}} \leq \min\left\{
\frac{1-q}{3\sqrt{2}\,CL\sqrt n},\,1
\right\}.\) We have
\begin{align}
& \frac{1}{T}\sum_{t=0}^{T-1}\EE{\norm{\nabfxbar}^2}
\le
\frac{4(f(\xbar^{(0)})-f^*)}{\sqrt{nT}}
+
\frac{2L\sigma^2}{\sqrt{nT}} \nonumber\\
&\qquad +
\frac{12L^2 C^2(\maxx)^2}{(1-q)^2T} 
+
\frac{12L^2 C^2 n^2\sigma^2}{(1-q)^2T}
+
\frac{36L^2 C^2 n^2\zeta^2}{(1-q)^2T} \nonumber \\
&=
\frac{2}{\sqrt{nT}}\left(2f(\xbar^{(0)})-2f^*
+ L\sigma^2\right) \nonumber\\
&\qquad\qquad +\frac{12L^2C^2}{(1-q)^2T}
\left((\maxx)^2+n^2\sigma^2+3n^2\zeta^2\right). \label{eq:convergence proof}
\end{align}

Let \(A:=2f(\xbar^{(0)})-2f^*+L\sigma^2\) and \(S:=(\maxx)^2+n^2\sigma^2+3n^2\zeta^2\).
To achieve $\epsilon$-convergence, i.e., $\frac{1}{T}\sum_{t=0}^{T-1}\EE{\norm{\nabfxbar}^2}\leq \epsilon$, 
it suffices to choose
\begin{align}
T \ge \max\Biggl\{
n,
\frac{18C^2L^2n^2}{(1-q)^2},
\frac{16}{n\epsilon^2}A^2,
\frac{24L^2C^2}{(1-q)^2\epsilon}S
\Biggr\}, \label{eq:convergence proof - T bound}
\end{align}
where the first two bounds ensure \(\sqrt{\frac{n}{T}} \leq \min\left\{
\frac{1-q}{3\sqrt{2}\,CL\sqrt n},\,1
\right\}\), and the last two bounds ensure that each of the terms in \eqref{eq:convergence proof} is bounded by $\epsilon/2$. 
If
\begin{align*}
\frac{2(1-q)^2A^2}{3L^2C^2nS}
\le
\epsilon
\le
\min\left\{
\frac{4}{3n^2}S,\;
\frac{24L^2C^2}{(1-q)^2}\frac{S}{n}
\right\},
\end{align*}
then the fourth term in \eqref{eq:convergence proof - T bound} dominates the other three terms. Hence, it suffices to satisfy
\[
T \ge
\frac{24L^2C^2}{(1-q)^2\epsilon}S.
\]
\end{proof}

\vspace{-1em}
\subsubsection{Supporting Lemmas}

\begin{lemma}[Consensus bound; Corollary 4 in \cite{Zhou25TCNS}]\label{lem: consensus bound} Under column-stochastic and Assumption (4). Then for \(t \ge \Delta B\),
\begin{equation}
\begin{aligned}
\norm{\xbart-\zit}
\le
&{\sqrt{n}}q^{t}
\left(\sum_{i=1}^n \|\bm{x}_i^{(0)}\|^2\right)^{1/2}\\
&+
\frac{\eta C}{\sqrt n}\sum_{s=0}^t q^{t-s}
\left(\sum_{i=1}^n \norm{\gis}^2\right)^{1/2}.
\end{aligned}
\end{equation}
where
\[C := \frac{4}{\delta^{\Delta B}},
\qquad
q := (1-\delta^{\Delta B})^{\frac{1}{\Delta B}}\in (0,1)\]
\end{lemma}
\begin{lemma}[Max form of the consensus bound]\label{lem: max consensus} Under the same assumptions as Lemma~\ref{lem: consensus bound}, for \(t \ge \Delta B\),
\begin{align}
\norm{\xbart-\zit}
\le
Cq^{t }\maxx
+
\eta C \sum_{s=0}^t q^{t-s}\maxgs.
\end{align}
\end{lemma}
\begin{proof}
By the inequality
\[
\left(\sum_{i=1}^n a_i^2\right)^{1/2} \le \sqrt{n}\max_i a_i,
\]
we have
\[
\left(\sum_{i=1}^n\|\bm{x}_i^{(0)}\|^2\right)^{1/2}
\le
\sqrt{n}\maxx,
\]
and for each \(s\),
\[
\left(\sum_{i=1}^n \norm{\gis}^2\right)^{1/2}
\le
\sqrt{n}\maxgs.
\]
Substituting these into the previous lemma cancels the factor \(1/\sqrt{n}\) and yields the result.
\end{proof}

\begin{lemma}[Maximum local gradient bound]\label{lem: max debias}
Under Assumptions~(1) and~(3), for any \(t\ge 0\),

\begin{align}
\EE{\max_j \norm{\nabfjt}^2}
\le
&3L^2 \sum_{i=1}^{n} \EE{\norm{\zit-\xbart}^2}
+
3n\zeta^2\notag\\
&+
3\EE{\norm{\nabfxbar}^2}.
\end{align} 
\end{lemma}
\begin{proof}

\begin{align*}
\norm{\nabfit}
\leq
&\norm{\nabfit - \nabfixbar}
+
\norm{\nabfixbar - \nabfxbar}\\
&+
\norm{\nabfxbar}
\end{align*}
\begin{align*}
\norm{\nabfit}^2
\leq
&3\norm{\nabfit - \nabfixbar}^2
+
3\norm{\nabfixbar - \nabfxbar}^2\\
&+
3\norm{\nabfxbar}^2
\end{align*}
Applied assumption (1):
\begin{align*}
\norm{\nabfit}^2
\leq
&3L^2\norm{\zit-\xbart}^2
+
3\norm{\nabfixbar - \nabfxbar}^2\\
&+
3\norm{\nabfxbar}^2
\end{align*}
Taking the maximum over node \(j\) on both sides gives
\begin{align*}
\max_j \norm{\nabfjt}^2
\le
&3L^2 \max_j \norm{\zjt-\xbart}^2
+
3\max_j \norm{\nabla f_j(\xbart)-\nabla f(\xbart)}^2\\
&+
3\norm{\nabfxbar}^2.
\end{align*}
Apply Assumption (3)
\[
\begin{aligned}
\max_j \norm{\zjt-\xbart}^2
&\le
\sum_{i=1}^{N}\norm{\zit-\xbart}^2\le
n\zeta^2\\
\end{aligned}
\]
\[
\max_j \norm{\nabfjt}^2
\le
3L^2 \sum_{i=1}^{n}\norm{\zit-\xbart}^2
+
3n\zeta^2
+
3\norm{\nabfxbar}^2.
\]
Taking expectation on both sides yields the lemma:
\[
\EE{\max_j \norm{\nabfjt}^2}
\le
3L^2 \sum_{i=1}^{n} \EE{\norm{\zit-\xbart}^2}
+
3n\zeta^2
+
3\EE{\norm{\nabfxbar}^2}.
\]
\end{proof}
\begin{lemma}[Consensus error recursion]
\label{lem: big Q}
Let
\(
Q_i^{(t)}:= \EE{\norm{\xbart - \zit}^2}.
\)
Under Assumptions~(1)–(4), for all \(t\ge 0\),
\begin{align}
\qit
&\le
3C^2 q^{2t}(\maxx)^2
+
\frac{3\eta^2 C^2 n\sigma^2}{(1-q)^2} 
+
\frac{9\eta^2 C^2 n\zeta^2}{(1-q)^2}\notag\\
&+
\frac{9\eta^2 C^2 L^2}{1-q}\sum_{s=0}^t q^{t-s}\sum_{i=1}^{n}\qis +
\frac{9\eta^2 C^2}{1-q}
\sum_{s=0}^t q^{t-s}\EE{\norm{\nabla f(\xbars)}^2}.
\end{align}
\end{lemma}

\begin{proof}
\[
\qit=\EE{\norm{\xbart - \zit}^2}
\]
Apply lemma~\ref{lem: max consensus}
\[
\qit\leq\EE{\left(Cq^{t }\maxx
+
\eta C \sum_{s=0}^t q^{t-s}\maxgs\right)^2}
\]

Using
\[
\max_j \norm{\gjs}
\le
\max_j \norm{\gjs-\nabfjs}
+
\max_j \norm{\nabfjs},
\]
which follows from triangle inequality, we obtain
\[
\qit
\leq \EE{\left(\begin{aligned}
&\underbrace{Cq^{t }\maxx}_{a}
+
\underbrace{\eta C \sum_{s=0}^t q^{t-s}\max_j \norm{\gjs-\nabfjs}}_{b}\\
&+\underbrace{\eta C \sum_{s=0}^t q^{t-s}\max_j \norm{\nabfjs}}_{c}
\end{aligned}
\right)^2}
\]
We have
\begin{align*}
\qit \leq \EE{3a^2+3b^2+3c^2}
= 3\EE{a^2}+3\EE{b^2}+3\EE{c^2}.
\end{align*}
Where
\[
\begin{aligned}
a&:=Cq^t\maxx,\\
b&:=\eta C\sum_{s=0}^t q^{t-s}\max_j \norm{\gjs-\nabfjs},\\
c&:=\eta C\sum_{s=0}^t q^{t-s}\max_j \norm{\nabfjs}.
\end{aligned}
\]
We have
\[
\EE{a^2}=C^2q^{2t}(\maxx)^2.
\vspace{-1em}
\]
\begin{align*}
b^2
&=
\eta^2 C^2 \left(\sum_{s=0}^t q^{t-s}\max_j \norm{\gjs-\nabfjs}\right)^2 \\
&=
\eta^2 C^2 \left(\sum_{s=0}^t \sqrt{q^{t-s}}
\sqrt{q^{t-s}}\,\max_j \norm{\gjs-\nabfjs}\right)^2\\
\overset{\text{Cauchy}}&{\leq}
\eta^2 C^2
\left(\sum_{s=0}^t q^{t-s}\right)
\left(\sum_{s=0}^t q^{t-s}\max_j \norm{\gjs-\nabfjs}^2\right) \\
&\leq
\frac{\eta^2 C^2}{1-q}
\sum_{s=0}^t q^{t-s}\max_j \norm{\gjs-\nabfjs}^2.
\end{align*}
\begin{align*}
\EE{b^2}
&\leq
\frac{\eta^2 C^2}{1-q}
\sum_{s=0}^t q^{t-s}\EE{\max_j \norm{\gjs-\nabfjs}^2} \\
&\leq
\frac{\eta^2 C^2}{1-q}
\sum_{s=0}^t q^{t-s}\sum_{j=1}^{n} \EE{\norm{\gjs-\nabfjs}^2} \\
\overset{\text{Assumption } (2)}&{\leq}
\frac{\eta^2 C^2 n\sigma^2}{1-q}
\sum_{s=0}^t q^{t-s} \\
&\leq
\frac{\eta^2 C^2 n\sigma^2}{(1-q)^2}.
\end{align*}
Similarly,
\begin{align*}
c^2
&=
\eta^2 C^2 \left(\sum_{s=0}^t q^{t-s}\max_j \norm{\nabfjs}\right)^2 \\
&=
\eta^2 C^2 \left(\sum_{s=0}^t \sqrt{q^{t-s}}
\sqrt{q^{t-s}}\,\max_j \norm{\nabfjs}\right)^2 \\
\overset{\text{Cauchy}}&{\leq}
\eta^2 C^2
\left(\sum_{s=0}^t q^{t-s}\right)
\left(\sum_{s=0}^t q^{t-s}\max_j \norm{\nabfjs}^2\right) \\
&\leq
\frac{\eta^2 C^2}{1-q}
\sum_{s=0}^t q^{t-s}\max_j \norm{\nabfjs}^2.\\
\end{align*}
Then,
\begin{align*}
\EE{c^2}
&\leq
\frac{\eta^2 C^2}{1-q}
\sum_{s=0}^t q^{t-s}\EE{\max_j \norm{\nabfjs}^2}\\
\overset{\text{lemma~\ref{lem: max debias}}}&{\leq}
\frac{\eta^2 C^2}{1-q}
\sum_{s=0}^t q^{t-s}
\left(
3L^2\sum_{i=1}^{n}\EE{\norm{\zis-\xbars}^2}
+
3n\zeta^2\right.
\\
& \hspace{14em} \left.+
3\EE{\norm{\nabla f(\xbars)}^2}
\right)\\
&\leq
\frac{3\eta^2 C^2 L^2}{1-q}
\sum_{s=0}^t q^{t-s}\sum_{i=1}^{n}Q_i^{(s)}
+
\frac{3\eta^2 C^2 n\zeta^2}{1-q}\sum_{s=0}^t q^{t-s} \\
&\hspace{10em} +
\frac{3\eta^2 C^2}{1-q}
\sum_{s=0}^t q^{t-s}\EE{\norm{\nabla f(\xbars)}^2}\\
&\leq
\frac{3\eta^2 C^2 L^2}{1-q}
\sum_{s=0}^t q^{t-s}\sum_{i=1}^{n}Q_i^{(s)}
+
\frac{3\eta^2 C^2 n\zeta^2}{(1-q)^2} \\
&\hspace{10em} +
\frac{3\eta^2 C^2}{1-q}
\sum_{s=0}^t q^{t-s}\EE{\norm{\nabla f(\xbars)}^2}.
\end{align*}
Combine all the term, we have
\begin{align*}
\qit
&\leq
3C^2 q^{2t}(\maxx)^2
+
\frac{3\eta^2 C^2 n\sigma^2}{(1-q)^2} +
\frac{9\eta^2 C^2 n\zeta^2}{(1-q)^2} \\
&+ 
\frac{9\eta^2 C^2 L^2}{1-q}\sum_{s=0}^t q^{t-s}\sum_{i=1}^{n}\qis +
\frac{9\eta^2 C^2}{1-q}
\sum_{s=0}^t q^{t-s}\EE{\norm{\nabla f(\xbars)}^2}
\end{align*}
\end{proof}

\begin{lemma}[Average consensus error bound]
\label{lem: avg Q}
Define \(M^{(t)} := \frac{1}{n} \sum_{i=1}^{n} Q_i^{(t)}\). Let \(
P:=1-\frac{9\eta^2 C^2 L^2 n}{(1-q)^2}\). Under the same conditions as lemma~\ref{lem: big Q} and for $P > 0$, we have 
\begin{align}
\sum_{t=0}^{T-1} M^{(t)}
&\le
\frac{3C^2}{P(1-q)^2}(\maxx)^2
+
\frac{3\eta^2 C^2 n\sigma^2}{P(1-q)^2}T \notag\\ 
&\quad
+
\frac{9\eta^2 C^2 n\zeta^2}{P(1-q)^2}T
+
\frac{9\eta^2 C^2}{P(1-q)^2}
\sum_{t=0}^{T-1}\EE{\norm{\nabla f(\xbart)}^2}.
\end{align}
\end{lemma}
\begin{proof}
\begin{align*}
\mt
&:= \frac{1}{n} \sum_{i=1}^{n} \qit \\
&\overset{\text{Lemma~\ref{lem: big Q}}}{\le}
3C^2 q^{2t}(\maxx)^2
+
\frac{3\eta^2 C^2 n\sigma^2}{(1-q)^2}
+
\frac{9\eta^2 C^2 n\zeta^2}{(1-q)^2} \\
&+
\frac{9\eta^2 C^2 L^2}{1-q}\frac{1}{n}\sum_{i=1}^{n}\sum_{s=0}^t q^{t-s}\sum_{i=1}^{n}\qis
+
\frac{9\eta^2 C^2}{1-q}
\sum_{s=0}^t q^{t-s}\EE{\norm{\nabla f(\xbars)}^2}\\
&\leq
3C^2 q^{2t}(\maxx)^2
+
\frac{3\eta^2 C^2 n\sigma^2}{(1-q)^2}
+
\frac{9\eta^2 C^2 n\zeta^2}{(1-q)^2} \\
&+
\frac{9\eta^2 C^2 L^2 n}{1-q}\sum_{s=0}^t q^{t-s}\ms
+
\frac{9\eta^2 C^2}{1-q}
\sum_{s=0}^t q^{t-s}\EE{\norm{\nabla f(\xbars)}^2}.
\end{align*}
Now Summing from t=0 to T-1, we have:
\begin{align*}
\sum_{t=0}^{T-1} \mt
&\le
3C^2(\maxx)^2 \sum_{t=0}^{T-1} q^{2t}
+
\frac{3\eta^2 C^2 n\sigma^2}{(1-q)^2}T
+
\frac{9\eta^2 C^2 n\zeta^2}{(1-q)^2}T \\
&+
\frac{9\eta^2 C^2 L^2 n}{1-q}
\sum_{t=0}^{T-1}\sum_{s=0}^t q^{t-s}\ms
\\
&\hspace{0em} +
\frac{9\eta^2 C^2}{1-q}
\sum_{t=0}^{T-1}\sum_{s=0}^t q^{t-s}\EE{\norm{\nabla f(\xbars)}^2}.
\end{align*}
Using the geometric-series bounds:
\[
\sum_{t=0}^{T-1} q^{2t}\le \frac{1}{1-q^2}, \qquad 
\sum_{t=0}^{T}\sum_{s=0}^{t}q^{t-j}\beta^{(s)}
\le
\frac{1}{1-q}\sum_{s=0}^{T}\beta^{(s)},
\]
valid for \(0<q<1\) and \(\beta^{(s)}\ge 0\), we obtain
\begin{align*}
\sum_{t=0}^{T-1} \mt
&\le
\frac{3C^2}{1-q^2}(\maxx)^2
+
\frac{3\eta^2 C^2 n\sigma^2}{(1-q)^2}T
+
\frac{9\eta^2 C^2 n\zeta^2}{(1-q)^2}T \\
&\quad
+
\frac{9\eta^2 C^2 L^2 n}{(1-q)^2}\sum_{s=0}^{T-1} \ms
+
\frac{9\eta^2 C^2}{(1-q)^2}
\sum_{s=0}^{T-1}\EE{\norm{\nabla f(\xbars)}^2}.
\end{align*}
Now subtract
\(\frac{9\eta^2 C^2 L^2 n}{(1-q)^2}\sum_{t=0}^{T-1}\mt\) 
from both sides. Using
$\frac{1}{1-q^2}\le \frac{1}{(1-q)^2}$
we obtain
\begin{align*}
\left(1-\frac{9\eta^2 C^2 L^2 n}{(1-q)^2}\right)\sum_{t=0}^{T-1} \mt
&\leq
\frac{3C^2}{(1-q)^2}(\maxx)^2
+
\frac{3\eta^2 C^2 n\sigma^2}{(1-q)^2}T \\
&
+
\frac{9\eta^2 C^2 n\zeta^2}{(1-q)^2}T
+
\frac{9\eta^2 C^2}{(1-q)^2}
\sum_{t=0}^{T-1}\EE{\norm{\nabla f(\xbart)}^2}.
\end{align*}
Let \(P:=\left(1-\frac{9\eta^2 C^2 L^2 n}{(1-q)^2}\right)\). If \(\eta \le \frac{1-q}{3\sqrt{2}\,CL\sqrt n}\)
then \(P\ge \frac12\). Therefore,
\begin{align*}
\sum_{t=0}^{T-1} M^{(t)}
&\le
\frac{3C^2}{P(1-q)^2}(\maxx)^2
+
\frac{3\eta^2 C^2 n\sigma^2}{P(1-q)^2}T \\
&+
\frac{9\eta^2 C^2 n\zeta^2}{P(1-q)^2}T
+
\frac{9\eta^2 C^2}{P(1-q)^2}
\sum_{t=0}^{T-1}\EE{\norm{\nabla f(\xbart)}^2}.
\end{align*}
\end{proof}
\begin{lemma}(Cumulative inequality). 
\label{lem: orginal}
The same summation argument as in~\cite[Lemma~8, Eq.~(28)]{Assran19ICML} implies
\begin{align}\label{eq: cummulative inequality}
\frac{\eta}{2}\sum_{t=0}^{T-1}
\EE{\norm{\nabfxbar}^2}
+
\frac{\eta - L\eta^2}{2}
\sum_{t=0}^{T-1}
\EE{
\norm{\frac{1}{n}\sum_{i=1}^n\nabfit}^2}\notag\\
\le
f(\xbar^{(0)}) - f^*
+
\frac{L\eta^2\sigma^2}{2n}\,T
+
\frac{\eta L^2}{2}
\sum_{t=0}^{T-1} M^{(t)}.
\end{align}
where 
\(M^{(t)}\) is defined as in Lemma~\ref{lem: avg Q}. 
\end{lemma}
\begin{lemma}(Cumulative  gradient bound).
\label{lem: final avg bound}
Under the same conditions as Lemma~\ref{lem: avg Q}, we get
\begin{align}
\left(1-\frac{9\eta^2 L^2 C^2}{P(1-q)^2}\right)
\sum_{t=0}^{T-1}\EE{\norm{\nabfxbar}^2}
\leq
\frac{2(f(\xbar^{(0)})-f^*)}{\eta}
+
\frac{L\eta\sigma^2T}{n}\notag\\
+
\frac{3L^2 C^2(\maxx)^2}{P(1-q)^2} 
+
\frac{3\eta^2 L^2 C^2 n\sigma^2T}{P(1-q)^2}
+
\frac{9\eta^2 L^2 C^2 n\zeta^2 T}{P(1-q)^2}.
\end{align}
\end{lemma}
\begin{proof}

Substituting the bound on $\sum_{t=0}^{T-1} M^{(t)}$ from Lemma~\ref{lem: avg Q} into \eqref{eq: cummulative inequality} yields
\begin{align*}
&\frac{\eta}{2}\sum_{t=0}^{T-1}
\EE{\norm{\nabla f(\xbart)}^2}
+
\frac{\eta - L\eta^2}{2}
\sum_{t=0}^{T-1}
\EE{\left\|
\frac{1}{n}\sum_{i=1}^n \nabla f_i(\zbf_i^{(t)})
\right\|^2}\\
&\leq 
f(\xbar^{(0)}) - f^*
+
\frac{L\eta^2\sigma^2T}{2n}
+
\frac{3\eta L^2C^2\maxx^2}{2P(1-q)^2}
+
\frac{3\eta^3L^2 C^2 n\sigma^2 T}{2P(1-q)^2}\\
&\qquad \qquad \qquad +
\frac{9\eta^3 L^2 C^2 n\zeta^2T}{2P(1-q)^2}
+
\frac{9\eta^3 L^2 C^2}{2P(1-q)^2}
\sum_{t=0}^{T-1}\EE{\norm{\nabfxbar}^2}.
\end{align*}
For \(\eta \le 1/L\), the second term on the left-hand side is nonnegative. Hence, dropping it and dividing both sides by \(\frac{\eta}{2}\), we obtain
\begin{align*}
\sum_{t=0}^{T-1}\EE{\norm{\nabfxbar}^2}
\leq
&\frac{2(f(\xbar^{(0)})-f^*)}{\eta}
+
\frac{L\eta\sigma^2T}{n}
   \\
& +
\frac{3L^2 C^2(\maxx)^2}{P(1-q)^2}
+
\frac{3\eta^2 L^2 C^2 n\sigma^2T}{P(1-q)^2}
\\
& +
\frac{9\eta^2 L^2 C^2 n\zeta^2 T}{P(1-q)^2} \\
&+
\frac{9\eta^2 L^2 C^2}{P(1-q)^2}
\sum_{t=0}^{T-1}\EE{\norm{\nabfxbar}^2}.
\end{align*}
Thus,
\begin{align*}
\left(1-\frac{9\eta^2 L^2 C^2}{P(1-q)^2}\right)
\sum_{t=0}^{T-1}\EE{\norm{\nabfxbar}^2}
\leq
\frac{2(f(\xbar^{(0)})-f^*)}{\eta}
+
\frac{L\eta\sigma^2T}{n}\\
+
\frac{3L^2 C^2(\maxx)^2}{P(1-q)^2} 
+
\frac{3\eta^2 L^2 C^2 n\sigma^2T}{P(1-q)^2}
+
\frac{9\eta^2 L^2 C^2 n\zeta^2 T}{P(1-q)^2}.
\end{align*}
\end{proof}

\subsection{Other Supporting Proofs}\label{appendix:Proofs}
\begin{proof}[Proof of Collorary~\ref{cor:SGP convergence}]
We have
 \[
C=\frac{4}{\delta^{\Delta B}},
\qquad
q=(1-\delta^{\Delta B})^{1/(\Delta B)}.
\]
Let
\[
x:=\delta^{\Delta B}.
\]
When \(x\ll 1\), a first-order Taylor expansion gives
\[
(1-x)^{1/(\Delta B)}
=
1-\frac{x}{\Delta B}+O(x^2).
\]
Hence
\[
1-q
=
1-(1-x)^{1/(\Delta B)}
=
\frac{x}{\Delta B}+O(x^2),
\]
which implies
\[
\frac{1}{(1-q)^2}
=
O\left(\frac{\Delta^2 B^2}{x^2}\right)
=
O\left(\frac{\Delta^2 B^2}{\delta^{2\Delta B}}\right).
\]
On the other hand,
\[
C^2=\frac{16}{\delta^{2\Delta B}}.
\]
Therefore,
\[
\frac{C^2}{(1-q)^2}
=
O\!\left(
\frac{\Delta^2 B^2}{\delta^{4\Delta B}}
\right).
\]
Substituting this into the expression for \(\underline{T}\), we obtain
\[
\underline{T}
=
\frac{24L^2C^2}{(1-q)^2\epsilon}S
=
O\left(\frac{\Delta^2 B^2}{\delta^{4\Delta B}}\right),
\]
where the big-\(O\) notation hides parameters independent of the mixing
matrices.   
\end{proof}
\begin{proof}[Proof of lemma~\ref{lem: uniform weight}]
Fix an iteration $t$ and a node $j$. Let $U_t(j) := N_t^{+}(j)\cup\{j\}$ and
$\delta_t(j) := \min_{i\in U_t(j)} W^{(t)}_{ij}$. Under column-stochastic mixing, we have
\[
1=\sum_{i\in U_t(j)} W^{(t)}_{ij} \ge \sum_{i\in U_t(j)} \delta_t(j) = |U_t(j)| \delta_t(j),
\]
which implies 
\[
\min_{i\in U_t(j)} W^{(t)}_{ij} = \delta_t(j) \le \frac{1}{|U_t(j)|}=\frac{1}{d_t^{+}(j)+1}.
\]
Equality holds if and only if $W^{(t)}_{ij}=1/(d_t^{+}(j)+1) $ for all $i\in U_t(j)$. 
Then taking the minimum over all $t$ and all $j$ yields
\[
\delta
= \min_{t\in\{0,\dots,B-1\}} \! \min_{j\in V} \delta_t(j) \leq \min_{t\in\{0,\dots,B-1\}} \ \min_{j\in V} \frac{1}{d_t^{+}(j)+1},  
\]
achievable under uniform weight assignment. 
\end{proof}

\begin{proof}[Proof of lemma~\ref{lem: equally split}]
Within each $B$-iteration window with a fixed activated graph $G_a$, activating the same link  multiple times is suboptimal, because it does not change $B$ or $\Delta$, but may decrease $\delta$ or increase the required communication schedule length. Therefore, without loss of generality, we assume that no link is activated more than once within a $B$-iteration window, i.e., 
\[
E_s \cap E_t = \emptyset, \qquad \forall\, s \neq t,\; s,t \in \{0,\dots,B-1\}.
\]
Under this assumption, the neighbor sets are disjoint across iterations, i.e.,
$N_s^{+}(j)\cap N_t^{+}(j)=\emptyset$ for all $s\neq t$. Consequently,
$d_a^{+}(j) = \sum_{t=0}^{B-1} d_t^{+}(j)$.
It follows that
\[
\max_{t\in \{0,\ldots,B-1\}} d_t^{+}(j)
\;\ge\;
\frac{1}{B}\sum_{t=0}^{B-1} d_t^{+}(j)
=
\frac{d_a^{+}(j)}{B},
\]
and hence
\[
\max_{t\in \{0,\ldots,B-1\}} \max_{j\in V} d_t^{+}(j)
\;\ge\;
\left\lceil \frac{D^{+}_a}{B}\right\rceil,
\]
where the ceiling is because $d_t^+(j)$ is an integer. 
Therefore, 
\[
\max \delta \hspace{-.15em}
=
\hspace{-.15em}
\min_t \min_j \frac{1}{d_t^{+}(j)+1}
\hspace{-.15em}
=
\hspace{-.15em}
\frac{1}{\max_t \max_j d_t^{+}(j)+1}
\hspace{-.15em}\le\hspace{-.15em}
\frac{1}{\left\lceil D^+_a/B\right\rceil+1}.
\]
\end{proof}

\begin{proof}[Proof of lemma~\ref{lem: B=1}]
First, we prove that 
\begin{align}
\sum_{t=0}^{B-1}\tau(G^{(t)})\ \ge\ \tau(G_a). \label{eq:sum tau}
\end{align}
Our cost model in Section~\ref{subsec:Cost Models} measures the communication cost of an activated graph by the length of the optimal (i.e., shortest) communication schedule. Let $E_{t,s}$ denote the set of links of $G^{(t)}$ scheduled in the $s$-th transmission slot of iteration $t$ and $\{E_{t,s}\}_{s=1}^{\tau(G^{(t)})}$ denote an optimal communication schedule for $G^{(t)}$. Then concatenating these schedules for $t=0,\ldots,B-1$ provides a feasible communication schedule $\bigcup_{t=0}^{B-1}\{E_{t,s}\}_{s=1}^{\tau(G^{(t)})}$ of length $\sum_{t=0}^{B-1}\tau(G^{(t)})$ for $G_a$. Therefore, the length of the optimal communication schedule for $G_a$, i.e., $\tau(G_a)$, must be no greater than $\sum_{t=0}^{B-1}\tau(G^{(t)})$.

Next, we prove that 
\begin{align}
B
\left(\left\lceil \frac{D^+_a}{B}\right\rceil + 1\right)^{4\Delta B} \geq (1+D^+_a)^{4\Delta}. \label{eq:B=1 proof part2}
\end{align}
By Bernoulli's inequality, for any $u\ge 0$ and integer $B\ge 1$,
$(1+u)^B \ge 1+Bu$. Taking $u=D_a^+/B$ gives
\[
\left(1+\frac{D_a^+}{B}\right)^B \ge 1+D_a^+,
\]
and hence
\[
\left(\left\lceil \frac{D^+_a}{B}\right\rceil + 1\right)^{4\Delta B}
\ \ge\
\left(1+\frac{D_a^+}{B}\right)^{4\Delta B} \ge (1+D_a^+)^{4\Delta}.
\]
Combining the above with $B\ge 1$ yields \eqref{eq:B=1 proof part2}.  

Combining \eqref{eq:sum tau} with \eqref{eq:B=1 proof part2} yields
\[
F(B)\ \ge\ \tau(G_a)\Delta^2 (1+D_a^+)^{4\Delta},
\]
which is exactly $F(1)$. 
\end{proof}

\begin{proof}[Proof of Lemma~\ref{lem:Dc bound - broadcast}]
Consider any vertex $e:=(i,j)\in E_a$ in the conflict graph $G_c$. 
Under the half-duplex constraint, $(i,j)$ conflicts with all links whose receiver
is $i$ and all links whose transmitter is $j$. Hence, the number of conflicts
from this condition is
\(
d_a^-(i) + d_a^+(j).
\)
Under the interference constraint, $(i,j)$ further conflicts with
any link $(k,l)\in E_a$ with $k\neq i$ such that $\{k,j\}\in E$ or $\{i,l\}\in E$. Let $N_G(i)$ denote the neighborhood of node $i$ in $G$ (excluding $i$). 
The number of such conflicts is bounded by
\[
\sum_{k\in N_G(j)\setminus\{i\}} d_a^{+}(k) + \sum_{l\in N_G(i)\setminus\{j\}} d_a^{-}(l) + (d_a^-(j)-1),
\]
where the last term counts other activated incoming links to $j$. 
Combining the two conditions, the degree of $e$ in $G_c$ satisfies
\begin{align}
d_{c}(e)
\;\le\;
&d_a^-(i) + d_a^+(j) + \sum_{k\in N_G(j)\setminus\{i\}} d_a^{+}(k) + \sum_{l\in N_G(i)\setminus\{j\}} d_a^{-}(l) \nonumber\\
&\hspace{14em} + d_a^-(j)-1.
\end{align}
We can bound each term as follows:
\[
d_a^-(i) \le D_a^-, \qquad d_a^+(j) \le D_a^+, \qquad d_a^-(j)-1 \leq D_a^-,
\]
\[
\sum_{k\in N_G(j)\setminus\{i\}} d_a^{+}(k)
\le |N_G(j)\setminus\{i\}|\, D_a^+ \le (D-1)\, D_a^+,
\]
\[
\sum_{l\in N_G(i)\setminus\{j\}} d_a^{-}(l)
\le |N_G(i)\setminus\{j\}|\, D_a^- \le (D-1)\, D_a^-.
\]
Therefore,
\begin{align}
d_{c}(e)
&\le
D_a^- + D_a^+ + (D-1)D_a^+ + (D-1) D_a^- + D_a^- \notag\\
&=
(D+1)\bigl(D_a^+ + D_a^-\bigr).
\end{align}
Taking the maximum over all \(e\in E_a\) yields the result. 
\end{proof}

\begin{proof}[Proof of Theorem~\ref{thm:graph design guarantee}]
We first prove the strong connectivity of the constructed $G_a$. After Step~2, $T$ must be a connected undirected graph. We will show that the way Step~3 orients the edges must turn this graph into a strongly connected directed graph. To this end, we will show that the orientation of edges in each bridge-connected component $C$ will turn it into a strongly connected subgraph. Let $r$ denote the root of the DFS tree for $C$. We observe the following:
\begin{enumerate}
\item There must be a directed path from $r$ to any vertex $v$ in $C$ along the DFS tree since all the tree edges are oriented away from $r$. 
\item For any tree edge $(u,v)$ with $pre(u)<pre(v)$, there must be a directed path from $v$ to $u$, as $u$ and $v$ must lie on a cycle in $C$ (since $C$ is 2-edge-connected) which must contain a back edge from some descendant of $v$ to some ancestor of $u$. 
\item Repeatedly applying the argument in (2) implies that there is a directed path from any vertex $v$ in $C$ back to $r$. 
\end{enumerate}
Combining observations (1) and (3) proves that our edge orientation turns each bridge-connected component into a strongly connected subgraph. Then connecting these subgraphs with bidirected links as in line~\ref{Graph Design:17} must yield a strongly connected graph. 

We now prove the bound in \eqref{eq:obj bound}. When $K=0$, the graph $T$ in line~\ref{Graph Design:8} is a tree (i.e., the minimum-degree spanning tree), for which every edge is a bridge and each bridge-connected component is just a single vertex. In this case, the directed graph $G_a$ produced by Step~3 is just the bidirected version of $T$. This implies that its maximum in/out-degree $D^-_a$/$D^+_a$ and diameter $\Delta$ satisfy 
\begin{align}
D^-_a = D^+_a = D^*,\qquad \Delta = \Delta^*. \label{eq:graph design proof - 1}
\end{align}
Plugging \eqref{eq:graph design proof - 1} into \eqref{eq: G_a obj} yields the right-hand side of \eqref{eq:obj bound}. 
The proof completes by noting that optimizing $K$ can further reduce the achieved value of \eqref{eq: G_a obj}. 
\end{proof}

\subsection{Additional Evaluation Results}\label{appendix:Additional Evaluation}
\begin{figure}[ht]
\vspace{-1em}
    \centering
    \begin{subfigure}[b]{0.23\textwidth}
        \centering
        \includegraphics[height = 1.3in,width=1\textwidth]{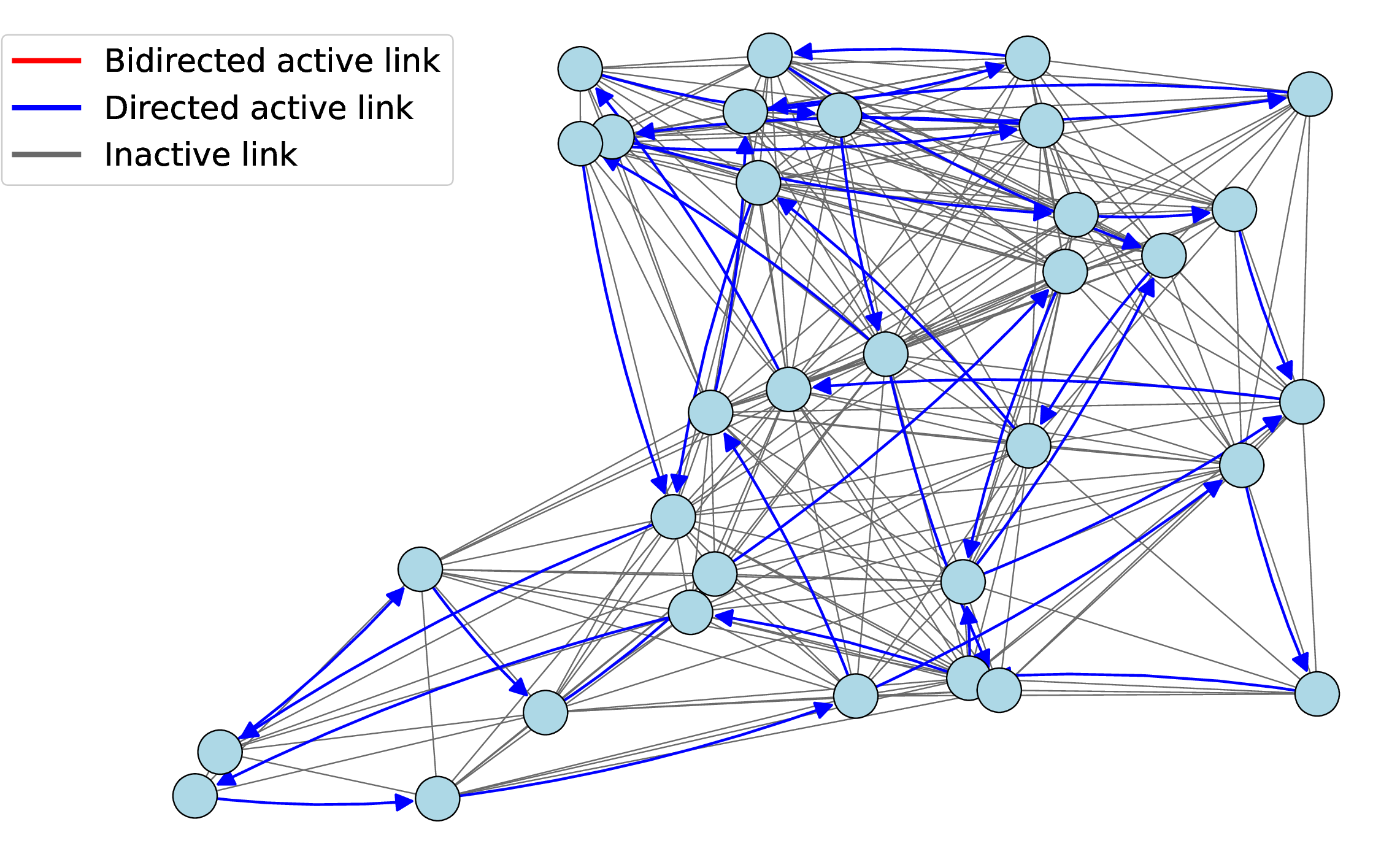}
        \vspace{-2em}
        \caption{RG}
        \label{fig:RG-a}
    \end{subfigure}
    \hfill
    \begin{subfigure}[b]{0.23\textwidth}
        \centering
        \includegraphics[height = 1.3in, width=1.1\textwidth]{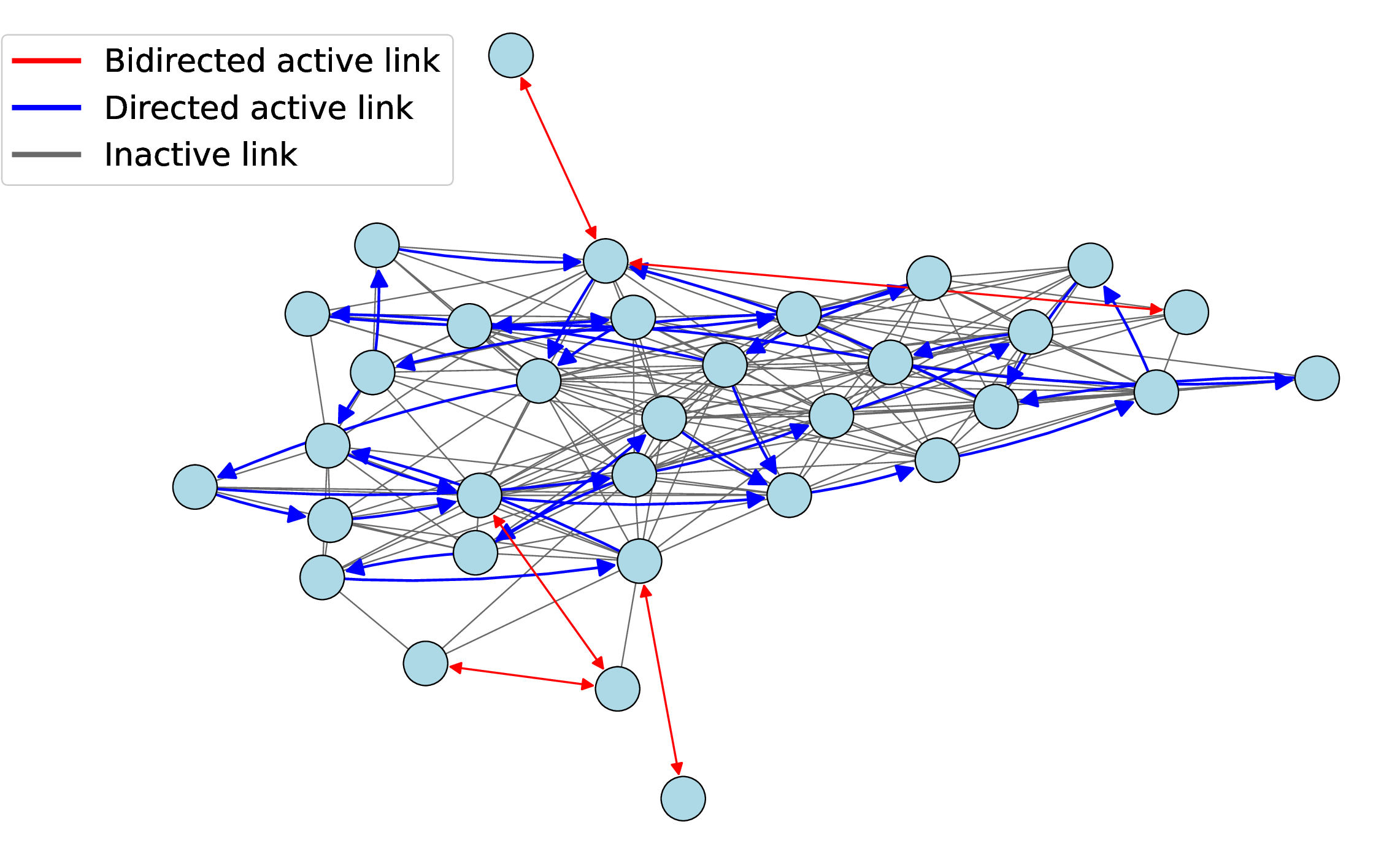}
         \vspace{-2em}
        \caption{Roofnet}
        \label{fig:Roofnet-b}
    \end{subfigure}
    \vspace{-0.75em}
    \caption{Designed communication graph without Step 4.}
    \label{fig:without step 4}
    \vspace{-1em}
\end{figure}            
As an ablation study to evaluate the value of Step~4 in Alg.~\ref{Alg:Graph Design}, we evaluate the learning performance of a variation of the proposed solution without this step (i.e., directly using $G_a$ obtained after Step~3 of Alg.~\ref{Alg:Graph Design} as the communication graph). Fig.~\ref{fig:without step 4} shows the resulting designs, which are much sparser than the communication graphs designed with Step~4 (Fig.~\ref{fig:graph_random} and \ref{fig:graph_roofnet}). The training results in Fig.~\ref{fig:results_cifar10_random_error_1}--\ref{fig:results_cifar10_roofnet_error_1} show that omitting the cost-preserving link augmentation in Step~4 causes a clear performance degradation in terms of both the logical convergence rate (in epochs) and the physical convergence rate (in slots). This observation validates the value of Step~4 in improving the actual learning performance, even though it is not guaranteed to improve the abstract design objective \eqref{eq: G_a obj}. 

\begin{figure}[ht]
    \centering    \centerline{\mbox{\includegraphics[height = 2.35in,width=1\linewidth]{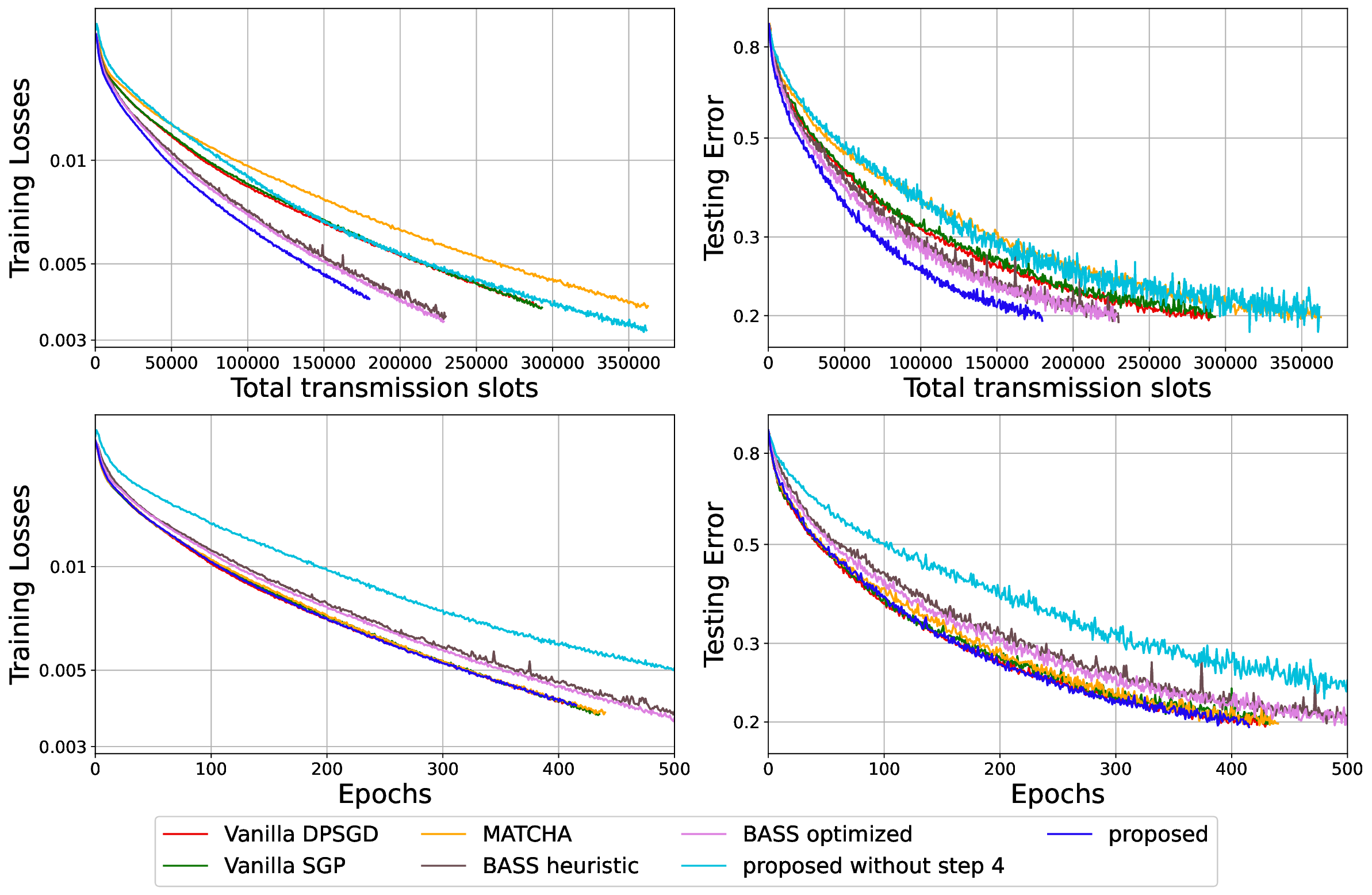}}}
    \vspace{-1em}
    \caption{Ablation study on RG.
        }

    \label{fig:results_cifar10_random_error_1}
    \vspace{-1em}
\end{figure}
\begin{figure}[ht]
    \centering    \centerline{\mbox{\includegraphics[height = 2.35in,width=1\linewidth]{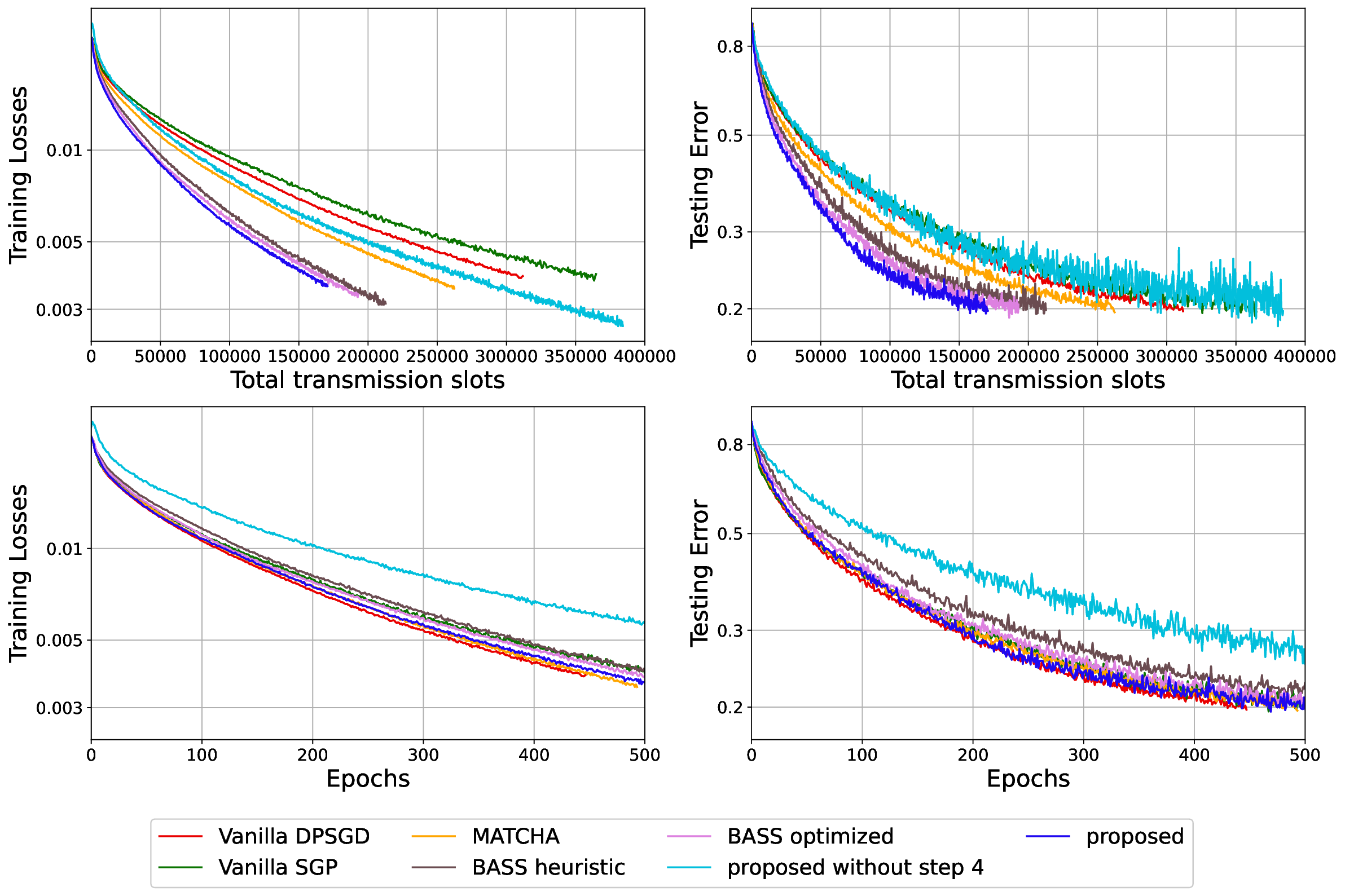}}}
    \vspace{-1em}
    \caption{Ablation study on Roofnet. 
        }

    \label{fig:results_cifar10_roofnet_error_1}
    \vspace{-.5em}
\end{figure}
\fi

\end{document}